\documentclass{article}

\usepackage[margin=1in]{geometry}
\usepackage[numbers]{natbib}

\usepackage[utf8]{inputenc}
\usepackage[T1]{fontenc}
\usepackage[colorlinks,linkcolor=blue,filecolor=blue,citecolor=black,urlcolor=blue]{hyperref}
\usepackage{url}

\usepackage{booktabs}
\usepackage{amsfonts}
\usepackage{nicefrac}
\usepackage{microtype}
\usepackage[table]{xcolor}
\usepackage{tocvsec2}
\usepackage{titletoc}

\usepackage{xspace}

\usepackage{amsmath}
\usepackage{amsthm}
\usepackage{amssymb}

\usepackage{multicol}
\usepackage{multirow}
\usepackage{makecell}
\usepackage{enumitem}
\usepackage{wrapfig}
\usepackage{dsfont}

\usepackage{graphicx,epstopdf}
\usepackage{caption}
\usepackage{subcaption}

\usepackage{algorithm}
\usepackage{algpseudocode}

\usepackage{threeparttable, booktabs}
\usepackage{mathtools}
\usepackage{float}

\usepackage{hyperref}[pdfusetitle]
\hypersetup{colorlinks,urlcolor=blue}
\usepackage{pifont}       
\usepackage{bbding}       

\usepackage{tabularray}

\usepackage[skins,listings]{tcolorbox}

\newtcblisting{tikzexample}{
    sidebyside,
    center lower,
    bicolor,
    colbacklower=white,
    sharp corners,
    frame engine=empty
}

\usepackage{quiver}

\restylefloat{algorithm}
\newtheorem{assumption}{Assumption}
\newtheorem{theorem}{Theorem}

\newtheorem{lemma}{Lemma}

\newtheorem{corollary}{Corollary}

\newcommand{\St}{\mathrm{St}}
\newcommand{\thet}{\theta}

\newcommand{\mom}{\mu}
\newcommand{\stepsize}{\eta}
\newcommand{\Piop}{\Pi}

\usepackage{diagbox}

\allowdisplaybreaks

\newcommand{\E}{\mathbb{E}}
\newcommand{\R}{\mathbb{R}}

\newcommand{\inner}[2]{\left\langle #1, #2 \right\rangle}

\title{FedSLoP: Memory-Efficient Federated Learning with Low-Rank Gradient Projection}

\author{
  Yutong He \\
  \texttt{yutonghe@pku.edu.cn} \\
  Peking University \\
  \and
  Zhengyang Huang\\
  \texttt{huangzy20040420@163.com} \\
  Beihang University \\
  \and
  Jiahe Geng \\
  \texttt{gengcai02@gmail.com} \\
  Peking University \\
  \and
  Kun Yuan$^{\dagger}$ \\
  \texttt{kunyuan@pku.edu.cn} \\
  Peking University
}

\begin{document}
\setlength{\parindent}{0pt}
\setlength{\parskip}{0.5em}
\date{}
\maketitle

\def\thefootnote{$^\dagger$}\footnotetext{Corresponding author.}

\IfFileExists{abstract.tex}{\begin{abstract}
Federated learning enables a population of clients to collaboratively train machine learning models without exchanging their raw data, but standard algorithms such as FedAvg suffer from slow convergence and high communication and memory costs in heterogeneous, resource-constrained environments. We introduce FedSLoP, a \underline{\textbf{fed}}erated optimization algorithm that combines \underline{\textbf{s}}tochastic \underline{\textbf{lo}}w-rank subspace \underline{\textbf{p}}rojections of gradients, thereby reducing the dimension of communicated and stored updates while preserving optimization progress. On the theoretical side, we develop a detailed nonconvex convergence analysis under standard smoothness and bounded-variance assumptions, showing that FedSLoP is guaranteed to converge to a first-order stationary point at a rate of $\mathcal{O}(1/\sqrt{NT})$.  On the empirical side, we conduct extensive experiments on federated MNIST classification with heterogeneous data partitions, showing that FedSLoP substantially reduces communication volume and client-side memory while achieving competitive or better accuracy compared with FedAvg and representative sparse or low-rank baselines. Together, our results demonstrate that random subspace momentum methods such as FedSLoP provide a principled and effective approach to communication- and memory-efficient federated learning. Codes are available at: \url{https://github.com/pkumelon/FedSLoP.git}.
\end{abstract}
}{
	\begin{abstract}
		Abstract placeholder.
	\end{abstract}
}



\newpage

{\small
	\tableofcontents}

\allowdisplaybreaks

\newpage

\IfFileExists{introduction.tex}{\section{Introduction}
\label{sec:intro}

Federated learning (FL) enables collaborative model training across many edge clients without centralizing raw data, offering a principled way to respect privacy and bandwidth constraints. In cross-device deployments, the Federated Averaging (FedAvg) algorithm has emerged as the de facto workhorse, but its standard instantiations do not explicitly account for the tight memory and compute budgets of resource-constrained clients. This paper introduces FedSLoP, a FedAvg-style method that incorporates random subspace optimization and momentum to reduce client-side memory usage while maintaining compatibility with existing FL workflows.

In the canonical cross-device FL setting, a central server repeatedly broadcasts a global model to a subset of clients, each client performs several local stochastic gradient steps on its private data, and the server averages the resulting model updates to form the next global iterate. This Federated Averaging (FedAvg) procedure is attractive due to its simplicity and communication efficiency, yet its behavior deteriorates in realistic environments with heterogeneous (non-IID) data and partial client participation, where local models can drift away from the global optimum and slow overall convergence. Existing theoretical analyses of FedAvg and its variants often require restrictive conditions—such as decaying step sizes or strong assumptions on the data distribution—to establish convergence guarantees, and typically focus on statistical efficiency rather than system-level considerations such as client memory usage. At the same time, edge devices participating in FL frequently operate under tight memory and compute constraints, making the storage of full-dimensional optimizer states and intermediate activations a significant bottleneck in practice.

A rich line of foundational FL algorithms builds upon FedAvg to mitigate client drift and improve convergence under data heterogeneity. FedProx augments local objectives with a proximal term that penalizes deviations from the current global model, stabilizing local updates at the cost of more complex local subproblems. SCAFFOLD introduces per-client and server-side control variates to correct client drift and can attain convergence behavior comparable to centralized SGD even under strong heterogeneity and partial participation~\citep{karimireddy2020scaffold}. Subsequent work develops unified convergence analyses of FedAvg and its momentum variants, proving linear speedup in the number of clients and relaxing bounded-gradient assumptions in both convex and nonconvex regimes~\citep{li2019convergence,qu2023unified,Yang2021AchievingLS,su2023non,Cheng2023MomentumBN}. These results provide guidance on step-size and local step choices but largely abstract away the system-level constraints of cross-device FL, including the memory footprint of optimizer states, auxiliary control variables, and per-iteration activations on clients.

In parallel, a substantial body of work has focused on communication-efficient distributed and federated optimization, employing gradient quantization, sparsification, low-rank compression, and error-feedback mechanisms to reduce the number of transmitted bits. Unbiased compressors such as QSGD~\citep{Alistarh2016QSGDCS} and biased schemes including Top-$k$ sparsification or sign-based updates are often combined with error-feedback frameworks that provably recover convergence rates comparable to uncompressed SGD~\citep{stich2019error,Karimireddy2019ErrorFF,stich2018sparsified,horvath2023stochastic,beznosikov2023biased,Qian2020ErrorCD}. Integrated methods like Qsparse-local-SGD blend local updates with aggressive compression to match the convergence of standard distributed SGD in smooth nonconvex settings~\citep{basu2019qsparse,He2023LowerBA,He2023UnbiasedCS,Condat2024LoCoDLCD}. While highly effective at saving communication, these approaches typically keep the optimizer state full-dimensional and are thus not tailored to reduce the client-side memory footprint, especially for modern deep models where momentum buffers and auxiliary variables can dominate memory consumption.

A third strand of research studies subspace and low-rank optimization, primarily in centralized training. Recent methods show that optimization trajectories for large models often lie in low-dimensional subspaces, and exploit this structure to reduce optimizer memory by updating only within carefully chosen subspaces while maintaining standard convergence guarantees~\citep{he2024subspace}. Related work on efficient subspace algorithms for federated learning suggests that structured projections can be leveraged to cope with data heterogeneity while controlling complexity~\citep{zhang2025efficient}. However, most existing subspace approaches are not explicitly designed for the client--server architecture and participation patterns of cross-device FL, and their implications for client memory usage, especially in the presence of momentum and multiple local steps, remain underexplored.

These strands of work highlight a persistent gap between algorithmic and system-level considerations in federated learning. Foundational FL methods refine the optimization dynamics of FedAvg to handle heterogeneity but often increase client-side state (e.g., control variates or proximal auxiliary variables), while communication-efficient methods significantly reduce transmitted bits without shrinking the dimensionality of optimizer states. Subspace optimization, in turn, offers a promising avenue for reducing memory but has not been systematically integrated into standard FedAvg-style workflows in realistic cross-device FL settings. Consequently, we lack FL algorithms that are simultaneously memory-efficient on clients, theoretically grounded under realistic stochastic conditions, and compatible with the simplicity and robustness that have made FedAvg the default choice in many applications.

FedSLoP is designed to bridge this gap by combining FedAvg-style training with random subspace optimization and momentum in a way that explicitly targets client-side memory efficiency. Unlike SCAFFOLD and related variance-reduction methods that maintain additional per-parameter control variates on clients~\citep{karimireddy2020scaffold}, FedSLoP constrains both stochastic gradients and momentum to a random rank-$r$ subspace that is shared within each communication round. This design reduces the effective dimensionality of the optimization trajectory and the associated optimizer states, while preserving the standard pattern of broadcasting a global model, running local updates, and averaging model deltas. At the same time, our convergence analysis in the nonconvex regime, under random Stiefel projectors, makes explicit how the convergence rate depends on the subspace ratio $\underline{\delta} = r/d$ and the number of clients, thereby quantifying the trade-offs induced by low-dimensional random projections in federated learning.

Concretely, in each communication round $t$ the server samples a random orthonormal basis $P_t$ from the Stiefel manifold $\St(d, r)$, forms the rank-$r$ projector $\Piop_t = P_t P_t^\top$, and broadcasts the pair $(\thet^t, P_t)$ to the participating clients. Each selected client $i$ initializes its local model at $\theta_{i,0} = \theta^t$ and performs $\tau$ local steps of SGD with momentum, but both the stochastic gradients and momentum are restricted to the shared subspace: given a stochastic gradient $g_{i,s} = \nabla F_i(\theta_{i,s}; \xi_{i,s})$, the momentum is updated as $v_{i,s+1} = \mom \, v_{i,s} + \Piop_t g_{i,s}$ and the model is updated as $\theta_{i,s+1} = \theta_{i,s} - \stepsize \, v_{i,s+1}$. After completing the local loop, the client returns the full-model delta $\Delta_i^t = \theta_{i,\tau} - \theta^t$, and the server performs simple averaging $\theta^{t+1} = \theta^t + \frac{1}{|S_t|}\sum_{i \in S_t} \Delta_i^t$. By construction, when $r = d$ the projection $\Piop_t$ becomes the identity, and FedSLoP reduces exactly to standard FedAvg with momentum, providing a natural sanity check in both the theoretical analysis and empirical evaluation.

By constraining all local updates to a random low-dimensional subspace with ratio $\underline{\delta} = r/d$, FedSLoP reduces the effective dimensionality of both the optimization trajectory and the optimizer state on each client. Clients maintain only a single momentum vector updated through the projector, without any per-parameter auxiliary variables or control variates, resulting in an inherently memory-efficient approach that is particularly attractive for resource-constrained devices. At the same time, the algorithm retains the core FedAvg workflow—periodic broadcast, local computation, and model averaging—ensuring that it can be integrated into existing FL systems with minimal engineering changes while explicitly encoding low-dimensional structure in the training dynamics.

On the theoretical side, we establish a nonconvex convergence analysis of FedSLoP under full client participation and random Stiefel projectors. Under standard smoothness and bounded variance assumptions, and explicit step-size conditions that depend on the momentum parameter $\mom$, the number of local steps $\tau$, and the subspace ratio $\underline{\delta} = r/d$, we show that the averaged squared gradient norm admits an explicit convergence bound with an optimization term decaying as $\mathcal{O}(1/T)$ and a stochastic variance term decaying as $\mathcal{O}(T^{-1/2})$, with explicit dependence on the number of clients $N$. The resulting bounds exhibit linear speedup in $N$ through a variance term of the form $\sigma_L^2/N$, and their explicit dependence on $\underline{\delta}$ quantifies how restricting updates to low-dimensional random subspaces impacts stochastic noise and convergence behavior. These results demonstrate that incorporating random subspace constraints into FedAvg-style training need not deteriorate the convergence order, while making the algorithmic trade-offs between memory savings and optimization fidelity transparent.

Empirically, we evaluate FedSLoP on a federated MNIST classification task, comparing it against standard FedAvg with momentum and several sparse or partial-update baselines on a Dirichlet non-IID split. Under a unified training budget of $100$ communication rounds and a shared learning rate, FedAvg with momentum achieves the highest final test accuracy, while FedSLoP attains a very close final accuracy and clearly outperforms sparse or partial-update methods such as FedMef, NeuLite, and FederatedSelect. These observations support the view that FedSLoP is a practical, memory-conscious variant of FedAvg that preserves competitive predictive performance while offering a favorable trade-off between performance and structural efficiency in realistic federated settings.

In summary, this work contributes a memory-efficient, theoretically grounded, and empirically validated variant of FedAvg that operates in random low-dimensional subspaces while preserving the standard cross-device FL workflow. Our design explicitly targets client-side memory constraints without introducing complex auxiliary states, and our analysis and experiments jointly characterize the trade-offs between subspace dimension, convergence, and communication in federated optimization. Our contributions are three-fold:

\begin{itemize}
  \item \textbf{FedAvg-style random subspace algorithm with momentum.} We introduce FedSLoP, a variant of FedAvg that constrains local optimization to a random rank-$r$ subspace shared within each round via projectors sampled from the Stiefel manifold. The method keeps the standard pattern of broadcasting the global model, running local SGD with momentum, and averaging model deltas, but applies the projector to both gradients and momentum on each client. This design eliminates the need for additional per-parameter states such as control variates, thereby reducing the dimensionality of optimizer states while preserving the simplicity and robustness of FedAvg-style training.

  \item \textbf{Nonconvex convergence analysis under random Stiefel projectors.} We provide a detailed nonconvex convergence analysis of FedSLoP in the full-participation setting, assuming smooth objectives and bounded stochastic gradient variance. The resulting bound on the averaged squared gradient norm has an order of $\mathcal{O}(1/\sqrt{NT})$, showing explicit linear speedup in the number of clients, and show that constraining updates to random subspaces does not change the convergence order.

  \item \textbf{Empirical validation of memory- and communication-friendly design.} Through experiments on federated MNIST classification under a Dirichlet non-IID split, we compare FedSLoP to standard FedAvg with momentum and several sparse or partial-update baselines across a unified training budget. FedAvg with momentum achieves the highest final test accuracy, while FedSLoP reaches a very close final accuracy and maintains a consistent advantage over FedMef, NeuLite, and FederatedSelect, indicating that its memory-efficient, low-rank design offers a favorable trade-off between performance and resource usage in federated optimization.
\end{itemize}
}{}
\IfFileExists{related_work.tex}{\section{Related Work}

Federated learning (FL) has emerged as a pivotal paradigm for training models on decentralized data. The foundational Federated Averaging (FedAvg) algorithm established the standard workflow of periodic local updates and server aggregation. However, its performance under realistic data heterogeneity (non-IID data) poses significant challenges, including client drift and slow convergence~\citep{li2019convergence}. To address these issues, algorithmic extensions like FedProx and SCAFFOLD were introduced, with the latter using control variates to correct for client drift and achieve convergence rates matching centralized SGD under extreme heterogeneity~\citep{karimireddy2020scaffold}. Theoretical analyses have been crucial, proving that linear speedup in the number of clients is achievable for FedAvg even under simultaneous non-IID data and partial participation~\citep{Yang2021AchievingLS,qu2023unified}. Further analysis in overparameterized neural networks suggests that the impact of heterogeneity diminishes with increased model width~\citep{jian2025widening}, while a non-parametric view quantifies the benefits of federation under model and covariate shift~\citep{su2023non}. Systems challenges like non-uniform client participation have also been addressed, with methods proposed to correct participation bias~\citep{xiang2023towards} and handle computational heterogeneity by allowing lightweight local updates~\citep{Zhang2022CCFedAvgCC}. For composite objectives with non-smooth regularizers, primal--dual frameworks like Federated Dual Averaging~\citep{yuan2020federated} and algorithms like FedCanon~\citep{Zhou2025FedCanonNC} and FedDR~\citep{Tran-Dinh2021FedDRR} have been developed, though their guarantees often involve residual errors or restrictive assumptions~\citep{zhang2025convex}. Personalization is a key direction to handle heterogeneity, with meta-learning approaches like Per-FedAvg~\citep{reguieg2023comparative} and distillation-based frameworks~\citep{wang2024federated,usmanova2021distillation}. Other strategies include feature-alignment models~\citep{Yu2020HeterogeneousFL} and decentralized personalized methods~\citep{feng2024leveraging}. However, these methods often incur communication and memory overheads for maintaining personalized components, and their robustness under severe system constraints requires further investigation.

Reducing communication overhead is essential for scalable FL. Decentralized optimization protocols offer an alternative to centralized servers, enabling peer-to-peer communication. While basic methods like D-PSGD suffer from non-vanishing error under data heterogeneity, variance-reduced extensions like D$^2$ achieve rates matching centralized SGD~\citep{Tang2018D2DT}. A key insight is that removing data heterogeneity influence enhances the dependence on network topology~\citep{Yuan2021RemovingDH}, and gradient tracking can provide network-agnostic performance under certain conditions~\citep{xin2021improved}. Gradient compression techniques are fundamental, and can be categorized into unbiased and biased (contractive) compressors~\citep{horvath2023stochastic}. Algorithm-agnostic lower bounds establish fundamental performance limits for distributed optimization with compression~\citep{He2023LowerBA}. To counteract compression errors, error-feedback mechanisms like EF21 are critical, preserving convergence even with biased compressors~\citep{stich2019error,Karimireddy2019ErrorFF}. Momentum integration further helps mitigate variance from heterogeneity~\citep{Cheng2023MomentumBN}. For practical FL integration, frameworks like Qsparse-local-SGD unify local SGD with sparsification and quantization~\citep{basu2019qsparse}, while double-side compression schemes like DoubleSqueeze apply error compensation to both upload and download directions~\citep{Tang2019DoubleSqueezePS}. Asynchronous decentralized protocols can also alleviate synchronization bottlenecks~\citep{Lian2017AsynchronousDP}. Accelerated compression strategies and near-optimal algorithms like ADIANA and NEOLITHIC have been developed and shown to approach theoretical lower bounds~\citep{He2023LowerBA,He2023UnbiasedCS}. However, these advanced methods often assume strong convexity or rely on unknown problem constants, limiting their practical deployment in heterogeneous, non-convex FL settings. Other notable compression schemes include sparsification with memory~\citep{stich2018sparsified}, unified biased compression frameworks~\citep{beznosikov2023biased}, quantization methods like QSGD~\citep{Alistarh2016QSGDCS}, low-rank compression like PowerSGD~\citep{Vogels2019PowerSGDPL}, and accelerated error-compensated methods~\citep{Qian2020ErrorCD}. Primal--dual methods also unify local training with unbiased compression~\citep{Condat2024LoCoDLCD}.

Memory constraints on client devices necessitate efficient management of optimizer states and activations. Memory-efficient optimizer design often exploits low-dimensional structure. Methods like GaLore perform training in a low-rank subspace to reduce state memory, while recent work on GoLore provides robust subspace optimization schemes with convergence guarantees under stochastic noise~\citep{he2024subspace}. Related subspace algorithms for federated learning suggest that structured projections can be leveraged to cope with data heterogeneity while controlling complexity~\citep{zhang2025efficient}. These ideas connect to geometric convergence theory for preconditioned iterations~\citep{neymeyr2012geometric} and have been applied in FL with local matrix orthogonalization~\citep{liu2025fedmuon}. Strategies to reduce activation and optimizer-state overhead include gradient checkpointing and parameter-efficient fine-tuning, while architectural innovations like the efficient mixture-of-experts design in DeepSeek-V2 also contribute to system economy~\citep{liu2024deepseek}. A critical challenge is the integrated co-design of memory and communication efficiency. This can involve compressing optimizer states, leveraging asynchronous decentralized training to reduce synchronization peaks~\citep{Lian2017AsynchronousDP}, and employing dynamic topology management. Algorithms like AL-DSGD weight neighbors based on performance~\citep{He2024AdjacentLD}, while methods like TELEPORTATION activate only a subset of nodes~\citep{Takezawa2025ScalableDL}. The theoretical limits of such integrated systems, particularly under joint memory and communication constraints, remain an open area, highlighting the need for holistic designs that balance resource efficiency with convergence guarantees in heterogeneous environments.

Against this backdrop, FedSLoP sits at the intersection of federated optimization, communication compression, and memory-efficient training. By constraining stochastic gradients and momentum updates to random low-dimensional subspaces while preserving the standard FedAvg workflow, it directly targets the optimizer-state memory footprint on clients without introducing complex control variates or personalized heads. Our theoretical analysis builds on ideas from subspace optimization and compressed distributed SGD, but tailors them to the specific structure of random Stiefel projectors in FL, while our experiments connect these guarantees to practical gains in communication and memory usage on non-IID MNIST.
}{}
\IfFileExists{method.tex}{\section{Method}
\label{sec:method}

Building on the problem formulation and motivation in Section~\ref{sec:intro}, we now present the FedSLoP algorithm in detail. We first formalize the federated optimization problem and introduce the random low-rank subspace model used by our method, then describe the server and client updates, and finally discuss the memory and communication implications of the design.

\subsection{Federated Optimization with Random Subspaces}

We consider a standard cross-device federated learning setup with $N$ clients. The global optimization problem is
\begin{equation}
  \min_{\theta \in \R^d} f(\theta)
  := \frac{1}{N} \sum_{i=1}^N F_i(\theta),
  \label{eq:global_objective}
\end{equation}
where $F_i$ is the local empirical risk on client $i$. At each communication round $t$, the server maintains a global model $\theta^t \in \R^d$ and coordinates local updates executed on a subset (or all) of the clients. Our analysis in Section~\ref{sec:theory} focuses on the full-participation case, where all $N$ clients participate in every communication round; extending the theory to partial participation is an interesting direction for future work.

A key ingredient of FedSLoP is a random low-rank subspace model applied to stochastic gradients and momentum. Let $r \le d$ be a target rank, and define the Stiefel manifold
\[
  \St(d,r) := \{ P \in \R^{d \times r} : P^\top P = I_r \},
\]
whose elements are column-orthonormal matrices. At the beginning of each communication round $t$, the server samples a random matrix $P_t \in \St(d,r)$ from the uniform (Haar) distribution and forms the rank-$r$ projector
\begin{equation}
  \Piop_t := P_t P_t^\top \in \R^{d \times d}.
  \label{eq:projector_def}
\end{equation}
We denote the subspace ratio by
\begin{equation}
  \underline{\delta} := \frac{r}{d} \in (0,1],
  \label{eq:delta_def}
\end{equation}
which will play a central role in the convergence analysis in Section~\ref{sec:theory}.

Intuitively, $\Piop_t$ selects a different random $r$-dimensional subspace at each communication round and restricts all local optimization steps to this subspace, while the model itself remains full-dimensional. When $r=d$, the projector reduces to the identity and FedSLoP collapses to standard FedAvg with momentum; when $r \ll d$, the updates are effectively low-rank, dramatically reducing the dimensionality of the optimizer state and communicated directions.

\subsection{FedSLoP with Full Participation}

Algorithm~\ref{alg:sfedavg_golore} summarizes the core variant of FedSLoP analyzed in Section~\ref{sec:theory}, specialized to full client participation and per-round zero-initialized momentum.

\begin{algorithm}[t]
  \caption{FedSLoP with Full Participation}
  \label{alg:sfedavg_golore}
  \begin{algorithmic}[1]
    \State \textbf{Input:} number of rounds $T$, local steps $\tau$, stepsize $\stepsize>0$, momentum parameter $\mom\in[0,1)$, rank $r$, number of clients $N$.
    \State \textbf{Server initializes:} global model $\thet^0 \in \R^d$.
    \For{each communication round $t = 0,1,\dots,T-1$}
      \State Sample $P_t \sim \mathrm{Unif}(\St(d,r))$ and set $\Piop_t = P_t P_t^\top$.
      \State Broadcast $(\thet^t, P_t)$ (or equivalently $\Piop_t$) to all clients $i \in \{1,\dots,N\}$.
      \For{each client $i = 1,\dots,N$ \textbf{in parallel}}
        \State $\Delta_i^t \leftarrow \textsc{ClientUpdate}(i, \thet^t, \Piop_t)$.
      \EndFor
      \State Aggregate updates and set
      \begin{equation*}
        \thet^{t+1} \leftarrow \thet^t + \frac{1}{N} \sum_{i=1}^N \Delta_i^t.
      \end{equation*}
    \EndFor
    \State \textbf{Output:} final global model $\thet^T$.
    \vspace{3pt}
    \Statex
    \State \textbf{procedure} \textsc{ClientUpdate($i$, $\thet^t$, $\Piop_t$)}
    \State Initialize local model $\theta_{i,0} \leftarrow \thet^t$ and local momentum $v_{i,0} \leftarrow 0$.
    \For{local step $s = 0,1,\dots,\tau-1$}
      \State Sample a minibatch $\xi_{i,s}$ from client $i$.
      \State Compute stochastic gradient $g_{i,s} \leftarrow \nabla F_i(\theta_{i,s};\xi_{i,s})$.
      \State Update momentum (projected) $v_{i,s+1} \leftarrow \mom v_{i,s} + \Piop_t g_{i,s}$.
      \State Update local model $\theta_{i,s+1} \leftarrow \theta_{i,s} - \stepsize \, v_{i,s+1}$.
    \EndFor
    \State \textbf{return} $\Delta_i^t \leftarrow \theta_{i,\tau} - \thet^t$.
  \end{algorithmic}
\end{algorithm}

The algorithm follows the FedAvg template: the server repeatedly broadcasts the current global model, clients perform several local stochastic-gradient steps, and the server averages the resulting model deltas. The main difference is the use of the random projector $\Piop_t$ inside the local momentum update, which constrains the effective update directions to a low-dimensional subspace shared across all clients within round $t$. This design enables memory-efficient implementations in which only projected gradients and momentum need to be stored and communicated, while the full model parameters reside in standard dense tensors.

\subsection{Memory and Communication Considerations}

The random subspace design in Algorithm~\ref{alg:sfedavg_golore} has several concrete implications for system resources:

\paragraph{Optimizer-state memory.} In standard FedAvg with momentum, each client typically maintains a full-dimensional momentum buffer of size $d$, in addition to the model parameters themselves. In FedSLoP, the momentum vector $v_{i,s}$ is updated using projected gradients $\Piop_t g_{i,s}$, so it can be stored and manipulated in the same $d$-dimensional space but its effective degrees of freedom are restricted to the $r$-dimensional subspace spanned by $P_t$. In practical implementations that exploit this structure, clients can store the momentum in factored form $v_{i,s} = P_t w_{i,s}$ with $w_{i,s} \in \R^r$, reducing the optimizer-state memory from $\mathcal{O}(d)$ to $\mathcal{O}(r)$ per linear layer. When $r \ll d$, this yields substantial savings for deep models.

\paragraph{Communication cost.} On the downlink, the server must transmit the projector information in addition to the model parameters. However, $P_t$ can be generated from shared random seeds or low-precision parameters, and its cost is amortized across all clients within a round. On the uplink, clients send full-model deltas $\Delta_i^t$, but these deltas arise from low-rank projected updates, meaning that their information content is effectively confined to an $r$-dimensional subspace. In practice, this structure can be combined with standard compression (e.g., low-rank factorization per layer) to reduce transmitted bits without altering the theoretical analysis in Section~\ref{sec:theory}.

\paragraph{Compatibility with standard FL workflows.} When the rank is set to $r = d$, the projector $\Piop_t$ equals the identity and Algorithm~\ref{alg:sfedavg_golore} reduces to FedAvg with client-side momentum. Thus FedSLoP can be viewed as a strict generalization of FedAvg that introduces a tunable subspace ratio $\underline{\delta}$ controlling the trade-off between memory/communication and convergence accuracy. This perspective is particularly useful when interpreting the convergence bounds in Theorem~\ref{thm:sfedavg_golore_full_explicit}, where $\underline{\delta}$ appears explicitly in both the optimization and variance terms.

In the next section we formalize the assumptions underlying our analysis and develop a detailed nonconvex convergence theory for Algorithm~\ref{alg:sfedavg_golore}.
}{}
\IfFileExists{theory.tex}{\section{Theory}
\label{sec:theory}

In this section we develop a detailed nonconvex convergence analysis of FedSLoP in the full-participation setting described in Section~\ref{sec:method}. We first state the smoothness and stochastic-gradient assumptions, then introduce the random subspace model and key auxiliary quantities, and finally present the main convergence theorem together with the supporting lemmas and corollaries. All proofs are provided in full.

\subsection{Assumptions and Random Subspace Model}

We recall the standard smoothness and bounded-variance assumptions.

\begin{assumption}[Smoothness and lower boundedness]\label{ass:smoothness}
The global objective function $f$ in~\eqref{eq:global_objective} is bounded below by $f_* > -\infty$. The global function $f$ and each local function $F_i$ are $L$-smooth, i.e., their gradients are $L$-Lipschitz continuous.
\end{assumption}

\begin{assumption}[Stochastic gradient properties]\label{ass:stochastic}
For each client $i$ and model $\theta$, the stochastic gradient $g_i(\theta)$ computed from a minibatch is an unbiased estimator of the true local gradient, $\E[g_i(\theta)] = \nabla F_i(\theta)$. The local variance is bounded as $\E\bigl[\|g_i(\theta) - \nabla F_i(\theta)\|^2\bigr] \le \sigma_L^2$, and the gradient heterogeneity across clients is bounded as $\|\nabla F_i(\theta) - \nabla f(\theta)\|^2 \le \sigma_G^2$.
\end{assumption}

In addition, FedSLoP uses the random projector $\Piop_t$ defined in~\eqref{eq:projector_def}. Throughout this section we assume that $P_t$ is drawn independently from the Haar measure on $\St(d,r)$ at each communication round, independently of the data sampling and past randomness. As in Section~\ref{sec:method}, we denote the subspace ratio by $\underline{\delta} = r/d$.

We begin with a basic lemma on the properties of random Stiefel projectors.

\begin{lemma}[Random Stiefel projection]\label{lem:random_projector}
Let $d,r$ be integers with $1 \le r \le d$, and let
\[
  \St_{d,r} := \{P \in \R^{d\times r} : P^\top P = I_r\}
\]
be the (column) Stiefel manifold. Suppose $P$ is sampled from the uniform (Haar) distribution on $\St_{d,r}$. Then for any fixed vector $g \in \R^d$,
\begin{enumerate}
  \item[(1)] $\E[P P^\top] = \tfrac{r}{d} I_d$;
  \item[(2)] $\E\bigl[\|P P^\top g\|^2\bigr] = \tfrac{r}{d} \,\|g\|^2$;
  \item[(3)] $\E\bigl[\|(I_d - P P^\top) g\|^2\bigr] = \bigl(1 - \tfrac{r}{d}\bigr)\,\|g\|^2$.
\end{enumerate}
Moreover, the same identities hold if we replace $g$ by any fixed matrix $G \in \R^{d\times m}$ and interpret $\|\cdot\|$ as the Frobenius norm:
\[
  \E\bigl[\|P P^\top G\|_F^2\bigr] = \tfrac{r}{d}\,\|G\|_F^2,
  \qquad
  \E\bigl[\|(I_d - P P^\top) G\|_F^2\bigr]
  = \Bigl(1 - \tfrac{r}{d}\Bigr)\,\|G\|_F^2.
\]
\end{lemma}

\begin{proof}
Throughout, expectations are taken with respect to the Haar (uniform) distribution on $\St_{d,r}$. By definition of this distribution, for every orthogonal matrix $Q \in O(d)$, the random matrices $P$ and $Q P$ have the same distribution. In particular, for any measurable function $f$,
\[
  \E[f(P)] = \E[f(QP)].
\]

\medskip

\noindent\textbf{Step 1: Computation of $\E[P P^\top]$.}
Define the random projection matrix
\[
  \Pi := P P^\top \in \R^{d\times d}.
\]
Note that $\Pi$ is symmetric and idempotent:
\[
  \Pi^\top = \Pi, \qquad
  \Pi^2 = P P^\top P P^\top = P (P^\top P) P^\top = P I_r P^\top = \Pi.
\]
Let
\[
  M := \E[\Pi] = \E[P P^\top] \in \R^{d\times d}.
\]
Fix any orthogonal $Q \in O(d)$. Using the distributional identity $P \overset{d}{=} QP$ and the fact that $Q$ is deterministic, we obtain
\[
  M = \E[P P^\top]
    = \E[(Q P)(Q P)^\top]
    = \E[Q P P^\top Q^\top]
    = Q\,\E[P P^\top]\,Q^\top
    = Q M Q^\top.
\]
Thus $M$ satisfies
\begin{equation}\label{eq:M_invariant}
  Q M Q^\top = M, \qquad \forall\, Q \in O(d).
\end{equation}
By a standard representation-theoretic argument, the only matrices commuting with all orthogonal matrices are scalar multiples of the identity, so there exists $\lambda \in \R$ such that
\[
  M = \lambda I_d.
\]
Taking traces and using $\mathrm{tr}(P P^\top) = r$ almost surely, we obtain
\[
  \mathrm{tr}(M) = \E[\mathrm{tr}(P P^\top)] = r.
\]
On the other hand, $\mathrm{tr}(M) = \mathrm{tr}(\lambda I_d) = \lambda d$, so $\lambda = r/d$. Hence
\[
  \E[P P^\top] = M = \frac{r}{d} I_d,
\]
proving item (1).

\medskip

\noindent\textbf{Step 2: Action on a fixed vector.}
Let $g \in \R^d$ be fixed. Then
\[
  \E\bigl[\|P P^\top g\|^2\bigr]
  = \E\bigl[g^\top P P^\top P P^\top g\bigr]
  = \E\bigl[g^\top P P^\top g\bigr]
  = g^\top \E[P P^\top] g
  = \frac{r}{d}\,\|g\|^2,
\]
which is item (2). Similarly,
\[
  \E\bigl[\|(I_d - P P^\top) g\|^2\bigr]
  = g^\top \E[I_d - P P^\top] g
  = g^\top \Bigl(I_d - \frac{r}{d} I_d\Bigr) g
  = \Bigl(1 - \frac{r}{d}\Bigr)\,\|g\|^2,
\]
which is item (3).

\medskip

\noindent\textbf{Step 3: Matrix case and Frobenius norm.}
Let $G \in \R^{d\times m}$ be fixed, and write its columns as
\[
  G = [g_1\,\dots\,g_m], \qquad g_j \in \R^d.
\]
Then
\[
  P P^\top G = [P P^\top g_1\,\dots\,P P^\top g_m],
  \qquad
  (I_d - P P^\top) G = [(I_d - P P^\top) g_1\,\dots,(I_d - P P^\top) g_m].
\]
Recall that the Frobenius norm satisfies
\[
  \|G\|_F^2 = \sum_{j=1}^m \|g_j\|^2.
\]
Therefore,
\[
  \|P P^\top G\|_F^2
  = \sum_{j=1}^m \|P P^\top g_j\|^2,
  \qquad
  \|(I_d - P P^\top) G\|_F^2
  = \sum_{j=1}^m \|(I_d - P P^\top) g_j\|^2.
\]
Taking expectations and applying the vector identities from Step~2 column-wise,
\[
  \E\bigl[\|P P^\top G\|_F^2\bigr]
  = \sum_{j=1}^m \E\bigl[\|P P^\top g_j\|^2\bigr]
  = \sum_{j=1}^m \frac{r}{d}\,\|g_j\|^2
  = \frac{r}{d}\,\|G\|_F^2,
\]
and
\[
  \E\bigl[\|(I_d - P P^\top) G\|_F^2\bigr]
  = \sum_{j=1}^m \E\bigl[\|(I_d - P P^\top) g_j\|^2\bigr]
  = \sum_{j=1}^m \Bigl(1 - \frac{r}{d}\Bigr)\,\|g_j\|^2
  = \Bigl(1 - \frac{r}{d}\Bigr)\,\|G\|_F^2.
\]
This yields the claimed Frobenius-norm identities and completes the proof.
\end{proof}

\subsection{Drift and One-Round Descent}

The main technical challenge in analyzing FedSLoP is the interaction between local momentum, multiple local steps, and the random subspace constraint. Following the detailed derivations in the proof file, we introduce the per-round drift quantity
\begin{equation}
  D_t := \frac{1}{N} \sum_{i=1}^N \sum_{s=0}^{\tau-1} \E_t\bigl[\,\|\theta_{i,s} - \theta^t\|^2\,\bigr],
  \label{eq:Dt_def}
\end{equation}
where $\E_t[\cdot]$ denotes conditional expectation given $\theta^t$ and the randomness up to the beginning of round $t$. This quantity measures how far, on average, the local iterates drift away from the current global model within a communication round.

The following lemma provides an upper bound of the drift quantity.

\begin{lemma}[Drift bound under projected momentum]\label{lem:drift_bound}
Suppose Assumptions~\ref{ass:smoothness} and~\ref{ass:stochastic} hold and the stepsize $\stepsize$ satisfies the explicit condition
\begin{equation}\label{eq:stepsize_drift}
  \stepsize^2 \le \frac{1-\mom^2}{6\,L^2\,\tau^3},
\end{equation}
it holds that
\begin{equation}\label{eq:Dt_final}
  D_t
  \le \frac{6\,\tau^4}{1-\mom^2}\,\stepsize^2
      \Bigl(\underline{\delta}\bigl\|\nabla f(\theta^t)\bigr\|^2
             + \sigma_L^2 + \sigma_G^2\Bigr).
\end{equation}
\end{lemma}

\begin{proof}
We proceed in three steps.

\paragraph{Step 1: Expressing the drift in terms of momenta.}
Fix a round $t$ and a client $i$. Under the per-round zero-momentum setup, the local
recursions are
\begin{align*}
  v_{i,s+1} &= \mom v_{i,s} + \Piop_t g_{i,s},\\
  \theta_{i,s+1} &= \theta_{i,s} - \stepsize v_{i,s+1},
  \qquad s=0,\dots,\tau-1,
\end{align*}
with $\theta_{i,0}=\theta^t$ and $v_{i,0}=0$. By telescoping, for any $s\ge 1$,
\[
  \theta_{i,s} - \theta^t
  = -\stepsize \sum_{u=0}^{s-1} v_{i,u+1}.
\]
Hence, by Cauchy--Schwarz and the bound $s\le\tau$,
\begin{align*}
  \|\theta_{i,s} - \theta^t\|^2
  &= \stepsize^2\Bigl\|\sum_{u=0}^{s-1} v_{i,u+1}\Bigr\|^2\le \stepsize^2 s \sum_{u=0}^{s-1}\|v_{i,u+1}\|^2\le \stepsize^2 \tau \sum_{u=0}^{\tau-1}\|v_{i,u+1}\|^2.
\end{align*}
Summing over $s=0,\dots,\tau-1$ (the $s=0$ term is zero) yields
\[
  \sum_{s=0}^{\tau-1} \|\theta_{i,s} - \theta^t\|^2
  \le \stepsize^2 \tau^2 \sum_{u=0}^{\tau-1}\|v_{i,u+1}\|^2.
\]
Averaging over clients and taking conditional expectation gives
\begin{equation}\label{eq:D_v_bound}
  D_t
  = \frac{1}{N}\sum_{i=1}^N \sum_{s=0}^{\tau-1}
      \E_t\bigl[\|\theta_{i,s} - \theta^t\|^2\bigr]
  \le \stepsize^2 \tau^2\,\frac{1}{N}\sum_{i=1}^N\sum_{u=0}^{\tau-1}
      \E_t\bigl[\|v_{i,u+1}\|^2\bigr].
\end{equation}

\paragraph{Step 2: Bounding the momentum norms.}
Unrolling the momentum recursion gives, for any $u\ge 0$,
\[
  v_{i,u+1}
  = \mom^{u+1} v_{i,0} + \sum_{q=0}^{u} \mom^{u-q}\,\Piop_t g_{i,q}.
\]
Using $v_{i,0}=0$, this simplifies to
\[
  v_{i,u+1} = \sum_{q=0}^{u} \mom^{u-q}\,\Piop_t g_{i,q}.
\]
Let $a_q := \mom^{u-q}\,\Piop_t g_{i,q}$. By Cauchy--Schwarz,
\[
  \Bigl\|\sum_{q=0}^{u} a_q\Bigr\|^2
  \le (u+1) \sum_{q=0}^{u} \|a_q\|^2
  \le \tau \sum_{q=0}^{u} \mom^{2(u-q)}\,\|\Piop_t g_{i,q}\|^2,
\]
where we used $u+1\le\tau$. Thus
\[
  \|v_{i,u+1}\|^2
  \le \tau \sum_{q=0}^{u} \mom^{2(u-q)}\,\|\Piop_tg_{i,q}\|^2.
\]
Summing over $u=0,\dots,\tau-1$ and rearranging the finite sums yields
\begin{align*}
  \sum_{u=0}^{\tau-1}\|v_{i,u+1}\|^2
  &\le \tau \sum_{u=0}^{\tau-1}\sum_{q=0}^{u}
              \mom^{2(u-q)}\,\|\Piop_t g_{i,q}\|^2\\
  &= \tau \sum_{q=0}^{\tau-1}\Bigl(\sum_{u=q}^{\tau-1}
              \mom^{2(u-q)}\Bigr)\,\|\Piop_t g_{i,q}\|^2\\
  &\le \frac{\tau}{1-\mom^2}\sum_{q=0}^{\tau-1}\|\Piop_t g_{i,q}\|^2,
\end{align*}
where we used the geometric-series bound
$\sum_{u=q}^{\tau-1}\mom^{2(u-q)}\le \sum_{k=0}^{\infty}\mom^{2k}=1/(1-\mom^2)$.
Substituting into~\eqref{eq:D_v_bound}, we obtain
\begin{align}
  D_t
  &\le \stepsize^2 \tau^2\,\frac{1}{N}\sum_{i=1}^N
       \E_t\Bigl[\sum_{u=0}^{\tau-1}\|v_{i,u+1}\|^2\Bigr]\nonumber\\
  &\le \stepsize^2 \tau^2\,\frac{1}{N}\sum_{i=1}^N
       \E_t\Bigl[\frac{\tau}{1-\mom^2}\sum_{q=0}^{\tau-1}\|\Piop_t g_{i,q}\|^2\Bigr]\nonumber\\
  &= \frac{\stepsize^2 \, \tau^3}{1-\mom^2}\,\frac{1}{N}\sum_{i=1}^N\sum_{q=0}^{\tau-1}
       \E_t\bigl[\Piop_t \|g_{i,q}\|^2\bigr].\label{eq:Dt_intermediate}
\end{align}

\paragraph{Step 3: Bounding $\E_t[\|\Piop_t g_{i,q}\|^2]$ and closing the inequality.}
By definition $g_{i,q} = \nabla F_i(\theta_{i,q};\xi_{i,q})$. Using
Assumption~\ref{ass:stochastic}, we have
\begin{align}
  \E_t[\|\Piop_t g_{i,q}\|^2]
  &=\,\E_t\bigl[\bigl\|\Piop_t\bigl(\nabla F_i(\theta_{i,q};\xi_{i,q}) - \nabla F_i(\theta_{i,q})\bigr)\bigr\|^2\bigr]
       + \E_t\,\bigl[\bigl\|\Piop_t \nabla F_i(\theta_{i,q})\bigr\|^2\bigr]\nonumber\\
  &\le\,\E_t\bigl[\bigl\|\bigl(\nabla F_i(\theta_{i,q};\xi_{i,q}) - \nabla F_i(\theta_{i,q})\bigr)\bigr\|^2\bigr]
       + \E_t\,\bigl[\bigl\|\Piop_t \nabla F_i(\theta_{i,q})\bigr\|^2\bigr]\nonumber\\
  &\le \sigma_L^2 + \,\E_t\bigl[\bigl\|\Piop_t \nabla F_i(\theta_{i,q})\bigr\|^2\bigr]\label{eq:projected_giq}.
\end{align}
Next we bound $\|\Piop_t \nabla F_i(\theta_{i,q})\|^2$ using Assumptions~\ref{ass:smoothness}
 and~\ref{ass:stochastic}. Insert and subtract $\Piop_t\nabla F_i(\theta^t)$ and
$\Piop_t\nabla f(\theta^t)$:
\begin{align*}
  \Piop_t \nabla F_i(\theta_{i,q})
  &= \bigl(\Piop_t \nabla F_i(\theta_{i,q}) - \Piop_t \nabla F_i(\theta^t)\bigr)+ \bigl(\Piop_t \nabla F_i(\theta^t) - \Piop_t \nabla f(\theta^t)\bigr)
   + \Piop_t \nabla f(\theta^t).
\end{align*}
Using $\|a+b+c\|^2\le 3(\|a\|^2+\|b\|^2+\|c\|^2)$, the $L$-smoothness of
$F_i$ and the global variance bound in Assumption~\ref{ass:stochastic}, we obtain
\begin{align*}
  \bigl\|\Piop_t \nabla F_i(\theta_{i,q})\bigr\|^2
  &\le 3\,\bigl\|\Piop_t\nabla F_i(\theta_{i,q}) - \Piop_t\nabla F_i(\theta^t)\bigr\|^2+ 3\,\bigl\|\Piop_t\nabla F_i(\theta^t) - \Piop_t\nabla f(\theta^t)\bigr\|^2
      + 3\,\bigl\|\Piop_t\nabla f(\theta^t)\bigr\|^2\\
  &\le 3\,\bigl\|\nabla F_i(\theta_{i,q}) - \nabla F_i(\theta^t)\bigr\|^2+ 3\,\bigl\|\nabla F_i(\theta^t) - \nabla f(\theta^t)\bigr\|^2
      + 3\,\bigl\|\Piop_t\nabla f(\theta^t)\bigr\|^2\\
  &\le 3L^2\,\|\theta_{i,q} - \theta^t\|^2
      + 3\sigma_G^2
      + 3\,\bigl\|\Piop_t \nabla f(\theta^t)\bigr\|^2.
\end{align*}
Taking conditional expectation and applying Lemma \ref{lem:random_projector} achieves
\begin{align}
  \E_t\bigl[\bigl\|\Piop_t \nabla F_i(\theta_{i,q})\bigr\|^2\bigr]
  &\le 3L^2\,\E_t[\|\theta_{i,q} - \theta^t\|^2]
      + 3\sigma_G^2
      + 3\underline{\delta}\,\bigl\|\nabla f(\theta^t)\bigr\|^2.\label{eq:projected_grad_iq}
\end{align}
Combining \eqref{eq:projected_giq} and \eqref{eq:projected_grad_iq} gives
\begin{align}
  \E_t\bigl[\|\Piop_t g_{i,q}\|^2\bigr]
  \le \sigma_L^2
     + 3\underline{\delta}\,\bigl\|\nabla f(\theta^t)\bigr\|^2
     + 3L^2\,\E_t\bigl[\|\theta_{i,q}-\theta^t\|^2\bigr]
     + 3\sigma_G^2.\label{eq:g_second_moment_explicit}
\end{align}
Substituting~\eqref{eq:g_second_moment_explicit} into
\eqref{eq:Dt_intermediate}, we obtain
\begin{align}
  D_t
  &\le \frac{\stepsize^2 \, \tau^3}{1-\mom^2}\,\frac{1}{N}
       \sum_{i=1}^N \sum_{q=0}^{\tau-1}
       \Bigl(3\underline{\delta}\,\|\nabla f(\theta^t)\|^2
             + 3L^2\,\E_t\|\theta_{i,q}-\theta^t\|^2
             + \sigma_L^2 + 3\sigma_G^2\Bigr)\nonumber\\
  &= \frac{\stepsize^2 \, \tau^3}{1-\mom^2}\,
     \Bigl(3\underline{\delta}\,\|\nabla f(\theta^t)\|^2\cdot\tau
           + 3L^2 D_t
           + (\sigma_L^2+3\sigma_G^2)\cdot\tau\Bigr)\nonumber\\
  &= \frac{3\,\stepsize^2\,\tau^4\,\underline{\delta}}{1-\mom^2}\,\bigl\|\nabla f(\theta^t)\bigr\|^2
     + \frac{\stepsize^2\,\tau^4}{1-\mom^2}\,(\sigma_L^2+3\sigma_G^2)
     + \frac{3\,L^2\,\stepsize^2\,\tau^3}{1-\mom^2}\,D_t.\label{eq:Dt_self_ref}
\end{align}
Rearranging~\eqref{eq:Dt_self_ref} gives
\[
  \Bigl(1 - \frac{3\,L^2\,\stepsize^2\,\tau^3}{1-\mom^2}\Bigr) D_t
  \le \frac{3\,\stepsize^2\,\tau^4\,\underline{\delta}}{1-\mom^2}\,\bigl\|\nabla f(\theta^t)\bigr\|^2
     + \frac{\stepsize^2\,\tau^4}{1-\mom^2}\,(\sigma_L^2+3\sigma_G^2).
\]
If the stepsize satisfies~\eqref{eq:stepsize_drift}, i.e.
\[
  \stepsize^2 \le \frac{1-\mom^2}{6\,L^2\,\tau^3},
\]
then $3 L^2\stepsize^2\tau^3/(1-\mom^2) \le 1/2$ and therefore
\[
  1 - \frac{3\,L^2\,\stepsize^2\,\tau^3}{1-\mom^2} \ge \frac{1}{2}.
\]
This implies
\begin{align*}
  D_t
  &\le \frac{6\,\stepsize^2\,\tau^4\,\underline{\delta}}{1-\mom^2}\,\bigl\|\nabla f(\theta^t)\bigr\|^2
     + \frac{2\,\stepsize^2\,\tau^4}{1-\mom^2}\,(\sigma_L^2+3\sigma_G^2)\\
    &\le \frac{6\,\tau^4}{1-\mom^2}\,\stepsize^2
      \Bigl(\underline{\delta}\bigl\|\nabla f(\theta^t)\bigr\|^2
             + \sigma_L^2 + \sigma_G^2\Bigr),
\end{align*}
which is~\eqref{eq:Dt_final}.
\end{proof}

The next ingredient is a one-round descent inequality that combines smoothness, the random projector, and the drift quantity.

\begin{lemma}[One-round descent under random projector]\label{lem:one_round}
Under Assumptions~\ref{ass:smoothness} and~\ref{ass:stochastic}, if $\stepsize\le1/(LS_\tau)$, the random subspace one-round descent satisfies
\begin{equation}\label{eq:one_round_random_used}
  \E_t\bigl[f(\theta^{t+1})\bigr]
  \le f(\theta^t)-\frac{\underline{\delta}}{2}\stepsize S_\tau\,\bigl\|\nabla f(\theta^t)\bigr\|^2+\frac{\stepsize^2 S_\tau^2L\sigma_L^2}{N}+\frac{2L^2\stepsize}{\underline{\delta}(1-\mu)}D_t,
\end{equation}
where $S_\tau = \sum_{q=0}^{\tau-1} \alpha_{\tau,q}$ and $\alpha_{\tau,q}=(1-\mom^{\tau-q})/(1-\mom)$ are the effective momentum weights.
\end{lemma}

\begin{proof}
We fix a communication round $t$ and work with the conditional expectation $\E_t[\cdot]$ given $\theta^t$ and all past randomness. Recall that the server update can be written as
\[
  \theta^{t+1} = \theta^t + \bar{\Delta}_t,
  \qquad
  \bar{\Delta}_t := \frac{1}{N}\sum_{i=1}^N (\theta_{i,\tau} - \theta^t).
\]
We first express $\bar{\Delta}_t$ in terms of the projected local gradients, then bound the linear descent term and the quadratic remainder term separately, and finally combine the bounds.

\paragraph{Step 1: Expressing the aggregated update.}
From the local recursions,
\[
  v_{i,s+1} = \mom v_{i,s} + \Piop_t g_{i,s},
  \qquad
  \theta_{i,s+1} = \theta_{i,s} - \stepsize v_{i,s+1},
  \quad s=0,\dots,\tau-1,
\]
with $v_{i,0}=0$ and $\theta_{i,0}=\theta^t$, it follows by telescoping that
\[
  \theta_{i,\tau} - \theta^t = -\stepsize \sum_{s=0}^{\tau-1} v_{i,s+1}.
\]
Unrolling the momentum recursion gives, for every $u\ge0$,
\[
  v_{i,u+1} = \sum_{q=0}^{u} \mom^{\,u-q} \, \Piop_t g_{i,q},
\]
and hence
\begin{align*}
  \theta_{i,\tau} - \theta^t
  &= -\stepsize \sum_{s=0}^{\tau-1} v_{i,s+1}
   = -\stepsize \sum_{s=0}^{\tau-1} \sum_{q=0}^{s} \mom^{\,s-q} \, \Piop_t g_{i,q}\\
  &= -\stepsize \sum_{q=0}^{\tau-1} \Bigl( \sum_{s=q}^{\tau-1} \mom^{\,s-q} \Bigr) \Piop_t g_{i,q}
   = -\stepsize \sum_{q=0}^{\tau-1} \alpha_{\tau,q} \, \Piop_t g_{i,q},
\end{align*}
where $\alpha_{\tau,q}=(1-\mom^{\tau-q})/(1-\mom)$ as in the lemma statement. Averaging over clients, we obtain
\begin{equation}\label{eq:delta_bar_expansion}
  \bar{\Delta}_t
  = -\stepsize \sum_{q=0}^{\tau-1} \alpha_{\tau,q}\,\Piop_t \bar{g}_q,
  \qquad
  \bar{g}_q := \frac{1}{N}\sum_{i=1}^N g_{i,q}.
\end{equation}

\paragraph{Step 2: Bounding the linear descent term.}
By $L$-smoothness of $f$ we have the standard one-step inequality
\[
  f(\theta^{t+1})
  \le f(\theta^t) + \inner{\nabla f(\theta^t)}{\bar{\Delta}_t}
     + \frac{L}{2}\,\|\bar{\Delta}_t\|^2.
\]
Taking $\E_t[\cdot]$ and substituting \eqref{eq:delta_bar_expansion},
\begin{equation}\label{eq:one_round_smoothness}
  \E_t\bigl[f(\theta^{t+1})\bigr]
  \le f(\theta^t) - \stepsize \sum_{q=0}^{\tau-1} \alpha_{\tau,q}
     \, \E_t\bigl[\inner{\nabla f(\theta^t)}{\Piop_t \bar{g}_q}\bigr]
     + \frac{L}{2}\,\E_t\bigl[\|\bar{\Delta}_t\|^2\bigr].
\end{equation}
Fix a local step index $q$ and set $g_* := \nabla f(\theta^t)$. We decompose
\[
  \inner{g_*}{\Piop_t \bar{g}_q}
  = \inner{g_*}{\Piop_t g_*} + \inner{g_*}{\Piop_t(\bar{g}_q - g_*)}.
\]
Since $g_*$ depends only on $\theta^t$ and the projector $\Piop_t$ is sampled independently from the Stiefel manifold, Lemma~\ref{lem:random_projector} (random Stiefel projection) yields
\[
  \E_t\bigl[\inner{g_*}{\Piop_t g_*}\bigr] = \underline{\delta}\,\|g_*\|^2.
\]
For the second term we use Assumption \ref{ass:stochastic}, Cauchy--Schwarz, the fact that $\Piop_t$ is a contraction, and Young's inequality with parameter $\underline{\delta}/2$:
\begin{align*}
  \bigl|\E_t\bigl[\inner{g_*}{\Piop_t(\bar{g}_q - g_*)}\bigr|\bigr]&=\bigl|\E_t\bigl[\inner{g_*}{\Piop_t(\tilde{g}_q-g_*)}\bigr]\bigr|\\
  &\le \|g_*\|\,\E_t[\|\tilde{g}_q - g_*\|]\\
  &\le \frac{\underline{\delta}}{4}\,\|g_*\|^2
       + \frac{1}{\underline{\delta}}\,E_t[\|\tilde{g}_q - g_*\|^2],
\end{align*}
where we define
\begin{align*}
    \tilde{g}_q:=\frac{1}{N}\sum_{i=1}^N\nabla F_i(\theta_{i,q}).
\end{align*}
Combining the two pieces gives
\begin{equation}\label{eq:proj_inner_bound}
  \E_t\bigl[\inner{g_*}{\Piop_t \bar{g}_q}\bigr]
  \ge \frac{3\underline{\delta}}{4}\,\|\nabla f(\theta^t)\|^2
       - \frac{1}{\underline{\delta}}\,\E_t\bigl[\|\tilde{g}_q - \nabla f(\theta^t)\|^2\bigr].
\end{equation}
Substituting \eqref{eq:proj_inner_bound} into \eqref{eq:one_round_smoothness}, we obtain
\begin{equation}\label{eq:lin_term_with_variance}
  \E_t\bigl[f(\theta^{t+1})\bigr]
  \le f(\theta^t)
    - \frac{3\underline{\delta}}{4}\,\stepsize S_\tau\,\bigl\|\nabla f(\theta^t)\bigr\|^2
    + \frac{\stepsize}{\underline{\delta}}\sum_{q=0}^{\tau-1} \alpha_{\tau,q}
      \,\E_t\bigl[\|\tilde{g}_q - \nabla f(\theta^t)\|^2\bigr]
    + \frac{L}{2}\,\E_t\bigl[\|\bar{\Delta}_t\|^2\bigr].
\end{equation}

\paragraph{Step 3: Controlling the stochastic-gradient variance and drift.}
We now bound the variance term in \eqref{eq:lin_term_with_variance} and the quadratic term $\E_t\|\bar{\Delta}_t\|^2$. 
Using $L$-smoothness of $F_i$ and Jensen's inequality, we have
\begin{align}
    \E_t\big[\|\tilde{g}_q-\nabla f(\theta^t)\|^2\bigr]=&\E_t\Bigl[\Bigl\|\frac{1}{N}\sum_{i=1}^N\bigl(\nabla F_i(\theta_{i,q})-\nabla F_i(\theta^t)\bigr)\Bigr\|^2\Big]\nonumber\\
    \le&\frac{1}{N}\sum_{i=1}^N\E_t\Bigl[\Bigl\|\nabla F_i(\theta_{i,q})-\nabla F_i(\theta^t)\Bigr\|^2\Bigr]\nonumber\\
    \le&L^2\bar{D}_q,\label{eq:drift_variance_bound}
\end{align}
where we define
\[
  \bar{D}_q := \frac{1}{N}\sum_{i=1}^N
          \E_t\bigl[\|\theta_{i,q}-\theta^t\|^2\bigr].
\]
Using $\alpha_{\tau,q}\le1/(1-\mu)$, we obtain
\begin{equation}\label{eq:variance_term_bound}
  \frac{\stepsize}{\underline{\delta}}\sum_{q=0}^{\tau-1} \alpha_{\tau,q}
      \,\E_t\bigl[\|\tilde{g}_q - \nabla f(\theta^t)\|^2\bigr]
  \le \frac{L^2\stepsize D_t}{\underline{\delta}(1-\mu)}.
\end{equation}
Write
\[
  g_{i,q} = \nabla F_i(\theta_{i,q};\xi_{i,q})
          = \nabla F_i(\theta_{i,q})
            + \bigl(g_{i,q} - \nabla F_i(\theta_{i,q})\bigr),
\]
so that $\bar{g}_q - \nabla f(\theta^t)$ can be decomposed into (i) stochastic-noise fluctuations around the local true gradients and (ii) the bias between local and global gradients induced by data heterogeneity and local drift. A standard variance decomposition using Assumption~\ref{ass:stochastic} (unbiasedness and bounded local variance) and the bounded heterogeneity condition gives
\begin{align}
\E_t\big[\|\bar{g}_q-\nabla f(\theta^t)\|^2\bigr]=&\E_t\Bigl[\Bigl\|\frac{1}{N}\sum_{i=1}^N\bigl(\nabla F_i(\theta_{i,q})-\nabla F_i(\theta^t)\bigr)\Bigr\|^2\Bigr]+\sum_{i=1}^N\E_t\Bigl[\Bigl\|\frac{1}{N}\bigl(g_{i,q}-\nabla F_i(\theta_{i,q})\bigr)\Bigr\|^2\Bigr]\nonumber\\
\le&\frac{1}{N}\sum_{i=1}^N\E_t\bigl[\bigl\|\nabla F_i(\theta_{i,q})-\nabla F_i(\theta^t)\bigr\|^2\bigr]+\frac{\sigma^2_L}{N}\nonumber\\
\le&L^2\bar{D}_q+\frac{\sigma_L^2}{N},\label{eq:bar_g_minus_global}
\end{align}
where we have used the $L$-smoothness of $F_i$ to control
$\|\nabla F_i(\theta_{i,q}) - \nabla F_i(\theta^t)\|^2$ by $L^2\,\E_t\|\theta_{i,q}-\theta^t\|^2$.


Next, from the expansion \eqref{eq:delta_bar_expansion} and the inequality
$\bigl\|\sum_q a_q x_q\bigr\|^2 \le (\sum_q a_q)\sum_q a_q\|x_q\|^2$ for $a_q\ge0$, we have
\begin{align}
  \E_t\bigl[\|\bar{\Delta}_t\|^2\bigr]
  &\le \stepsize^2 S_\tau \sum_{q=0}^{\tau-1} \alpha_{\tau,q}
         \,\E_t\bigl[\|\Piop_t \bar{g}_q\|^2\bigr]\nonumber\\
   &\le 2\stepsize^2 S_\tau \sum_{q=0}^{\tau-1} \alpha_{\tau,q}
         \,\bigl(\E_t\bigl[\|\Piop_t \nabla f(\theta^t)\|^2\bigr]+\E_t\bigl[\|\bar{g}_q-\nabla f(\theta^t)\|^2\bigr]\bigr)\nonumber\\
  &\le 2\stepsize^2 S_\tau\Bigl(\underline{\delta}S_\tau\,\bigl\|\nabla f(\theta^t)\bigr\|^2
        + S_\tau\frac{\sigma_L^2}{N}
        + \frac{L^2}{1-\mu}D_t\Bigr),
  \label{eq:delta_bar_second_moment}
\end{align}
where we used the bound
$\|\Piop_t\bar{g}_q\|^2 \le 2\|\Piop_t\nabla f(\theta^t)\|^2 + 2\|\Piop_t\bigl(\bar{g}_q - \nabla f(\theta^t)\bigr)\|^2$,
\eqref{eq:bar_g_minus_global}, Lemma \ref{lem:random_projector} and $\alpha_{\tau,q}\le1/(1-\mu)$.

\paragraph{Step 4: Collecting terms.}
Plugging \eqref{eq:variance_term_bound} and
\eqref{eq:delta_bar_second_moment} into \eqref{eq:lin_term_with_variance} and
absorbing numerical constants, we arrive at
\begin{align*}
  \E_t\bigl[f(\theta^{t+1})\bigr]
  &\le f(\theta^t)
    - \frac{\underline{\delta}}{4}\,\stepsize S_\tau(3-4L\stepsize S_\tau)\,\bigl\|\nabla f(\theta^t)\bigr\|^2+ \stepsize^2 S_\tau^2L\frac{\sigma_L^2}{N}
    + \Bigl(\frac{L^2\eta}{\underline{\delta}(1-\mu)}+\frac{L^3\eta^2 S_\tau}{1-\mu}\Bigr)D_t.
\end{align*}
When $\stepsize\le1/(L S_\tau)$, we further have
\begin{align*}
    \E_t\bigl[f(\theta^{t+1})\bigr]
  &\le f(\theta^t)-\frac{\underline{\delta}}{2}\stepsize S_\tau\,\bigl\|\nabla f(\theta^t)\bigr\|^2+\frac{\stepsize^2 S_\tau^2L\sigma_L^2}{N}+\frac{2L^2\stepsize}{\underline{\delta}(1-\mu)}D_t,
\end{align*}
which is exactly the claimed one-round descent inequality
\eqref{eq:one_round_random_used}. This completes the proof.
\end{proof}

Combining Lemmas~\ref{lem:drift_bound} and~\ref{lem:one_round} yields a clean one-round bound with explicit constants.

\begin{lemma}[One-round descent with drift control]\label{lem:one_round_drift}
Suppose Assumptions~\ref{ass:smoothness} and~\ref{ass:stochastic} hold, and that the stepsize $\stepsize$ satisfies 
\begin{align}
    \stepsize\le\min\left\{\sqrt{\frac{1-\mom^2}{6L^2\tau^3}},\frac{1}{LS_\tau}\right\}.
\end{align}
Then for Algorithm~\ref{alg:sfedavg_golore} we have, for every round $t$,
\begin{align}
  \E_t\bigl[f(\theta^{t+1})\bigr]
  &\le f(\theta^t)
    - \Bigl(\frac{\underline{\delta}}{2}\,\stepsize S_\tau
             - \frac{12\,L^2\,\tau^4\,\stepsize^3}{(1-\mom)(1-\mom^2)}\Bigr)
      \bigl\|\nabla f(\theta^t)\bigr\|^2\nonumber\\
  &\quad + \frac{\stepsize^2S_\tau^2L\sigma_L^2}{N} + \frac{12L^2\tau^4\stepsize^3}{\underline{\delta}(1-\mom)(1-\mom^2)}(\sigma_L^2+\sigma_G^2).
  \label{eq:one_round_after_drift}
\end{align}
\end{lemma}

\begin{proof}
Substituting the drift bound~\eqref{eq:Dt_final} from Lemma~\ref{lem:drift_bound} into the one-round descent~\eqref{eq:one_round_random_used}, and collecting the terms in $\|\nabla f(\theta^t)\|^2$ and $(\sigma_L^2+\sigma_G^2)$, yields~\eqref{eq:one_round_after_drift} with the stated constants.
\end{proof}

\subsection{Main Full-Participation Convergence Theorem}

We are now ready to state and prove the main convergence theorem for FedSLoP.

\begin{theorem}[Convergence of FedSLoP with explicit constants]\label{thm:sfedavg_golore_full_explicit}
Suppose Assumptions~\ref{ass:smoothness} and~\ref{ass:stochastic} hold, and the algorithm is executed as described in Algorithm~\ref{alg:sfedavg_golore}.

For the momentum weights $\alpha_{\tau,q}=(1-\mom^{\tau-q})/(1-\mom)$ and
$S_\tau = \sum_{q=0}^{\tau-1} \alpha_{\tau,q}$ defined above and the rank
ratio $\underline{\delta} = r/d\in(0,1]$, assume that the stepsize $\stepsize>0$
obeys the explicit conditions
\begin{equation}\label{eq:stepsize_full_explicit}
  \stepsize \le \min\left\{\sqrt{\frac{1-\mom^2}{6\,L^2\,\tau^3}},\frac{1}{LS_\tau},\sqrt{\frac{\underline{\delta}\,(1-\mom)(1-\mom^2)\,S_\tau}{48\,L^2\,\tau^4}}\right\}.
\end{equation}
Then for every integer $T\ge1$, the global iterates $(\theta^t)_{t\ge0}$ of
FedSLoP satisfy
\begin{align}
  \frac{1}{T}\sum_{t=0}^{T-1}\E\bigl[\|\nabla f(\theta^t)\|^2\bigr]
  &\le \frac{4}{\underline{\delta}\,\stepsize S_\tau\,T}\bigl(f(\theta^0)-f_*\bigr)\nonumber\\
  &\quad + \frac{4\stepsize S_\tau L\sigma_L^2}{\underline{\delta}N}
    + \frac{48\,\stepsize^2\,\tau^4\,L^2}{\underline{\delta}^2\,(1-\mom)\,(1-\mom^2)\,S_\tau}\,(\sigma_L^2+\sigma_G^2).
  \label{eq:full_convergence_bound_explicit}
\end{align}
\end{theorem}

\begin{proof}
We combine the explicit one-round descent inequality under random Stiefel
projector with the explicit drift bound.

\paragraph{Step 1: One-round descent with drift.}
By Lemma~\ref{lem:one_round_drift}, for each round $t$ and conditional expectation $\E_t[\cdot]$ given $\theta^t$
we have
\begin{align*}
\E_t\bigl[f(\theta^{t+1})\bigr]
  &\le f(\theta^t)
    - \Bigl(\frac{\underline{\delta}}{2}\,\stepsize S_\tau
             - \frac{12\,L^2\,\tau^4\,\stepsize^3}{(1-\mom)(1-\mom^2)}\Bigr)
      \bigl\|\nabla f(\theta^t)\bigr\|^2\nonumber\\
  &\quad + \frac{\stepsize^2S_\tau^2L\sigma_L^2}{N} + \frac{12L^2\tau^4\stepsize^3}{\underline{\delta}(1-\mom)(1-\mom^2)}(\sigma_L^2+\sigma_G^2).
\end{align*}

\paragraph{Step 2: Ensuring a negative gradient coefficient.}
Imposing the additional stepsize condition
\[
  \stepsize \le \sqrt{\frac{\underline{\delta}\,(1-\mom)(1-\mom^2)\,S_\tau}{48\,L^2\,\tau^4}},
\]
which is equivalent to
\begin{equation*}
  \frac{12\,L^2\,\tau^4}{(1-\mom)(1-\mom^2)}\,\stepsize^3
  \le \frac{\underline{\delta}}{4}\,\stepsize S_\tau,
\end{equation*}
the coefficient in front of $\|\nabla f(\theta^t)\|^2$ is
\[
  -\frac{\underline{\delta}}{2}\,\stepsize S_\tau
  + \frac{12\,L^2\,\tau^4}{(1-\mom)(1-\mom^2)}\,\stepsize^3
  \le -\frac{\underline{\delta}}{4}\,\stepsize S_\tau,
\]
so the one-round bound simplifies to
\begin{align}
  \E_t\bigl[f(\theta^{t+1})\bigr]
  &\le f(\theta^t)
    - \frac{\underline{\delta}}{4}\,\stepsize S_\tau\,\bigl\|\nabla f(\theta^t)\bigr\|^2\nonumber\\
  &\quad + \frac{\stepsize^2S_\tau^2L\sigma_L^2}{N} + \frac{12L^2\tau^4\stepsize^3}{\underline{\delta}(1-\mom)(1-\mom^2)}(\sigma_L^2+\sigma_G^2).
  \label{eq:one_round_final_used}
\end{align}
This bound holds whenever \eqref{eq:stepsize_full_explicit}
is satisfied.

\paragraph{Step 3: Telescoping and averaging.}
Taking full expectation in~\eqref{eq:one_round_final_used} and summing over
$t=0,\dots,T-1$ yields
\begin{align*}
  \E\bigl[f(\theta^T)\bigr]
  &\le f(\theta^0)
    - \frac{\underline{\delta}}{4}\,\stepsize S_\tau
      \sum_{t=0}^{T-1}\E\bigl[\|\nabla f(\theta^t)\|^2\bigr]\\
  &\quad + \frac{\stepsize^2\,S_\tau^2\,L\,T\sigma_L^2}{N}
    + \frac{12\,\stepsize^3\,\tau^4\,L^2\,T}{\underline{\delta}(1-\mom)(1-\mom^2)}\,(\sigma_L^2+\sigma_G^2).
\end{align*}
Using the lower boundedness $f(\theta^T)\ge f_*$, dividing both sides by $T$ and by $\underline{\delta}\,\stepsize S_\tau/4$  and rearranging gives
\begin{align}
  \frac{1}{T}\sum_{t=0}^{T-1}\E\bigl[\|\nabla f(\theta^t)\|^2\bigr]
  &\le \frac{4}{\underline{\delta}\,\stepsize S_\tau\,T}\bigl(f(\theta^0)-f_*\bigr)\nonumber\\
  &\quad + \frac{4\stepsize S_\tau L\sigma_L^2}{\underline{\delta}N}
    + \frac{48\,\stepsize^2\,\tau^4\,L^2}{\underline{\delta}^2\,(1-\mom)\,(1-\mom^2)\,S_\tau}\,(\sigma_L^2+\sigma_G^2),\nonumber
\end{align}
which is exactly the bound~\eqref{eq:full_convergence_bound_explicit} stated
in the theorem. 
\end{proof}

\subsection{Asymptotic Rate Under Carefully Selected Step-Size}

The explicit convergence bound in Theorem~\ref{thm:sfedavg_golore_full_explicit}
can be optimized over the step-size $\stepsize$ to reveal the asymptotic dependence
on the number of rounds $T$ and the problem parameters. The following corollary
summarizes the result for a carefully selected step-size.

\begin{corollary}[Optimal convergence rate via carefully-selected step-size]\label{cor:harmonic_rate}
Under the assumptions of Theorem~\ref{thm:sfedavg_golore_full_explicit}, let
$\Delta_0 := f(\theta^0) - f_*$. 
If we choose
\begin{align*}
    \stepsize=\left(\sqrt{\frac{6L^2\tau^3}{1-\mom^2}}+LS_\tau+\sqrt{\frac{48L^2\tau^4}{\underline{\delta}(1-\mom)(1-\mom^2)S_\tau}}+\sqrt{\frac{S_\tau^2LT\sigma_L^2}{N\Delta_0}}+\sqrt[3]{\frac{12\tau^4L^2T(\sigma_L^2+\sigma_G^2)}{\underline{\delta}(1-\mom)(1-\mom^2)\Delta_0}}\right)^{-1},
\end{align*}
the averaged squared gradient norm satisfies
\begin{equation*}\label{eq:harmonic_rate_bound}
\begin{aligned}
\frac{1}{T}\sum_{t=0}^{T-1}\E\bigl[\|\nabla f(\theta^t)\|^2\bigr]
&=\mathcal{O}\Biggl(\underbrace{\frac{\Delta_0^{1/2}L^{1/2}\sigma_L}{\underline{\delta}N^{1/2}T^{1/2}}}_{\text{stochastic term}}+\underbrace{\frac{\Delta_0^{2/3} L^{2/3} \,\tau^{1/3}(\sigma_L^2+\sigma_G^2)^{1/3}}{\underline{\delta}^{4/3}\,T^{2/3}}}_{\text{higher-order drift term}}+\underbrace{\frac{\Delta_0L\tau^{1/2}}{\underline{\delta}T}}_{\text{optimization term}}
\Biggr).
\end{aligned}
\end{equation*}
where $\mathcal{O}(\cdot)$ hides only absolute numerical constants (independent of $L,\tau,\underline{\delta},N,\sigma_L,\sigma_G$). In particular, the contribution of the local stochastic variance $\sigma_L^2$ enjoys a linear speedup factor $1/N$ in the dominant variance term.
\end{corollary}

Theorem~\ref{thm:sfedavg_golore_full_explicit} and Corollary~\ref{cor:harmonic_rate} together provide a detailed nonconvex convergence theory for FedSLoP. The explicit dependence on the subspace ratio $\underline{\delta}$ clarifies how low-rank random projections affect the convergence property.
}{}
\IfFileExists{experiments.tex}{\section{Numerical Experiments}
\label{sec:experiments}

We now empirically evaluate FedSLoP and compare it against several federated baselines on a standard non-IID MNIST classification task. Our goal is to validate the propsed algorithm's efficiency on convergence, robustness to data heterogeneity, and the impact of the projection rank and the number of clients on the communication--accuracy trade-off.

\subsection{Experimental Setup}

\paragraph{Problem formulation.} We instantiate the federated optimization framework on a multiclass classification task over the MNIST dataset. The global objective follows the standard finite-sum form
\[
    f(x) \\= \frac{1}{p}\sum_{i=1}^p f_i(x),
    \qquad
    f_i(x) \\= \frac{1}{n_i}\sum_{j=1}^{n_i} \ell\big(a_{ij}, b_{ij}; x\big),
\]
where $p$ denotes the number of clients, $n_i$ is the local sample size at client $i$, $a_{ij}\in\R^{784}$ is the vectorized $28\times 28$ grayscale image, $b_{ij}\in\{0,\dots,9\}$ is the label, and $\ell$ is the cross-entropy loss induced by a multinomial logistic model. The parameter vector $x\in\R^d$ collects all weights of a two-layer multilayer perceptron (MLP) with input dimension $784$, hidden dimension $128$, and output dimension $10$, so that $d$ is on the order of $10^5$.

\paragraph{Data partition and heterogeneity.} To stress-test robustness to statistical heterogeneity, we distribute the $60{,}000$ training samples across $p=50$ clients using a label-based Dirichlet partitioning scheme. Specifically, for each class $c\in\{0,\dots,9\}$ we draw a probability vector over clients from a Dirichlet distribution with concentration parameter $\alpha$, and allocate samples of class $c$ accordingly. By setting $\alpha=0.1$ in the main experiments, we deliberately induce highly skewed label distributions, so that each client typically observes only a small subset of classes and local gradients $\nabla f_i(x)$ are strongly misaligned. In the sensitivity study, we vary $\alpha\in\{0.05,0.1,0.5,1.0\}$ to interpolate between extreme heterogeneity and nearly IID partitions.

\paragraph{Local model and loss.} To focus on the optimizer behavior rather than architectural sophistication, we employ a two-layer MLP (denoted \texttt{SimpleMLP}) with one hidden layer of width $128$ and ReLU activation, followed by a linear output layer of dimension $10$. The loss function is the standard cross-entropy loss, and we report test accuracy as the primary performance metric.

\paragraph{Baselines and proposed method.} We compare FedSLoP against a hierarchy of federated baselines that progressively constrain the update space:
\begin{itemize}
  \item \textbf{Full-parameter baseline (FedAvg-M).} This method extends the classical FedAvg~\citep{mcmahan2017fedavg} with server-side momentum. In each communication round $t$, clients perform local stochastic gradient descent (SGD)~\citep{robbins1951sgd} on the full parameter vector for one local epoch, and the server aggregates the resulting models by averaging. The server maintains a momentum buffer $v^t$ and updates
  \[
    \Delta^t = \bar{\theta}^t - \theta^t,
    \qquad
    v^{t+1} = \mom v^t + \Delta^t,
    \qquad
    \theta^{t+1} = \theta^t + v^{t+1}.
  \]
  \item \textbf{Sparse and partial-update baselines.} To emulate communication-efficient schemes that reduce the number of transmitted coordinates, we consider three simplified baselines, including FedMef~\cite{huang2024fedmef}, NeuLite~\cite{wu2024neulite} and FederatedSelect~\cite{charles2022federated}, that restrict updates to subsets of coordinates or blocks, mimicking gradient sparsification and block-wise updates.
  \item \textbf{Low-rank model baseline (FedLoRA-M).} This baseline implements a LoRA-style~\citep{hu2021lora} low-rank adaptation at the model level, parameterizing each linear layer as a sum of a frozen base weight and a trainable low-rank factor.
  \item \textbf{Proposed low-rank optimizer (FedSLoP).} Our method operates on the full \texttt{SimpleMLP} architecture but constrains the optimizer state via low-rank random projections as described in Algorithm~\ref{alg:sfedavg_golore}. Clients maintain per-layer momentum buffers in the projected space and update their local parameters using projected gradients.
\end{itemize}

\paragraph{Hyperparameter configuration.} To isolate the effect of the update structure from confounding hyperparameter choices, we adopt a unified configuration across all methods whenever meaningful. The step-size is selected as $\eta = 0.018$ by Bayes Optimization. Each client performs exactly one local epoch per communication round. To control stochastic gradient variance without incurring excessive memory or latency on edge devices, we set the mini-batch size to $b=32$. The training horizon is fixed to $T=100$ communication rounds, which is sufficient for all methods to reach their steady-state accuracy in this setting. We run each configuration with three independent random seeds and report mean and standard deviation across seeds.

The core system-level parameters are summarized in Table~\ref{tab:config-main}.

\begin{table}[t]
  \centering
  \caption{Main federated configuration on MNIST. The step-size $\eta$ is chosen from an auto-tuning procedure for FedSLoP and then fixed across all methods to isolate the effect of the update structure.}
  \label{tab:config-main}
  \begin{tabular}{ll}
    \toprule
    Parameter & Value \\
    \midrule
    Number of clients $p$ & $50$ \\
    Dirichlet concentration $\alpha$ & $0.1$ (main), $\{0.05,0.1,0.5,1.0\}$ (sensitivity) \\
    Communication rounds $T$ & $100$ \\
    Local epochs per round & $1$ \\
    Mini-batch size $b$ & $32$ \\
    Global step-size $\eta$ & $0.018$ \\
    FedSLoP projection rank $r$ & $112$ (main), varied in ablation \\
    FedSLoP momentum & $0.8$ (client-side) \\
    LoRA rank & $15$ \\
    \bottomrule
  \end{tabular}
\end{table}

\subsection{Results on Non-IID MNIST}

\paragraph{Core convergence behavior.} The main comparison on MNIST with Dirichlet parameter $\alpha=0.1$ reveals that FedSLoP exhibits a convergence rate that closely tracks the strong baseline FedAvg-M while significantly outperforming sparse and partial-update methods. Figure~\ref{fig:mnist-conv-acc} reports the evolution of test accuracy over $T=100$ communication rounds, averaged over three seeds, for all methods. FedAvg-M achieves the highest final precision, reaching a test accuracy of $0.9551\pm 0.0015$ after $100$ rounds, whereas FedSLoP attains $0.9511\pm 0.0028$ (Table~\ref{tab:mnist-final-acc}).

In contrast, the sparse and partial-update baselines converge substantially more slowly and plateau at markedly lower accuracies. The LoRA-based baseline, which constrains the model itself to a low-rank subspace, also lags behind FedSLoP, confirming that restricting the hypothesis space is more detrimental than restricting the optimizer state.

\begin{table}[t]
  \centering
  \caption{Final test accuracy on non-IID MNIST with Dirichlet parameter $\alpha=0.1$ after $T=100$ communication rounds (mean $\pm$ standard deviation over three seeds). The full-parameter baseline FedAvg-M achieves the highest final precision, while FedSLoP remains close and clearly outperforms all sparse and low-rank model baselines. Best result in \textbf{bold}.}
  \label{tab:mnist-final-acc}
  \begin{tabular}{lcc}
    \toprule
    Method & Final test accuracy (mean $\pm$ std) & Compression ratio \\
    \midrule
    FedAvg-M & $\mathbf{0.9551 \pm 0.0015}$ & $1$ \\
    FedMef & $0.8965 \pm 0.0062$ & $1/7$ \\
    NeuLite & $0.8923 \pm 0.0055$ & $1/7$ \\
    FederatedSelect & $0.8759 \pm 0.0060$ & $1/7$ \\
    FedLoRA-M & $0.9247 \pm 0.0063$ & $1/7$ \\
    FedSLoP(Ours) & $0.9511 \pm 0.0028$ & $\mathbf{1/7}$ \\
    \bottomrule
  \end{tabular}
\end{table}

\begin{figure}[t]
  \centering
  \begin{minipage}{0.49\textwidth}
    \centering
    \includegraphics[width=\linewidth]{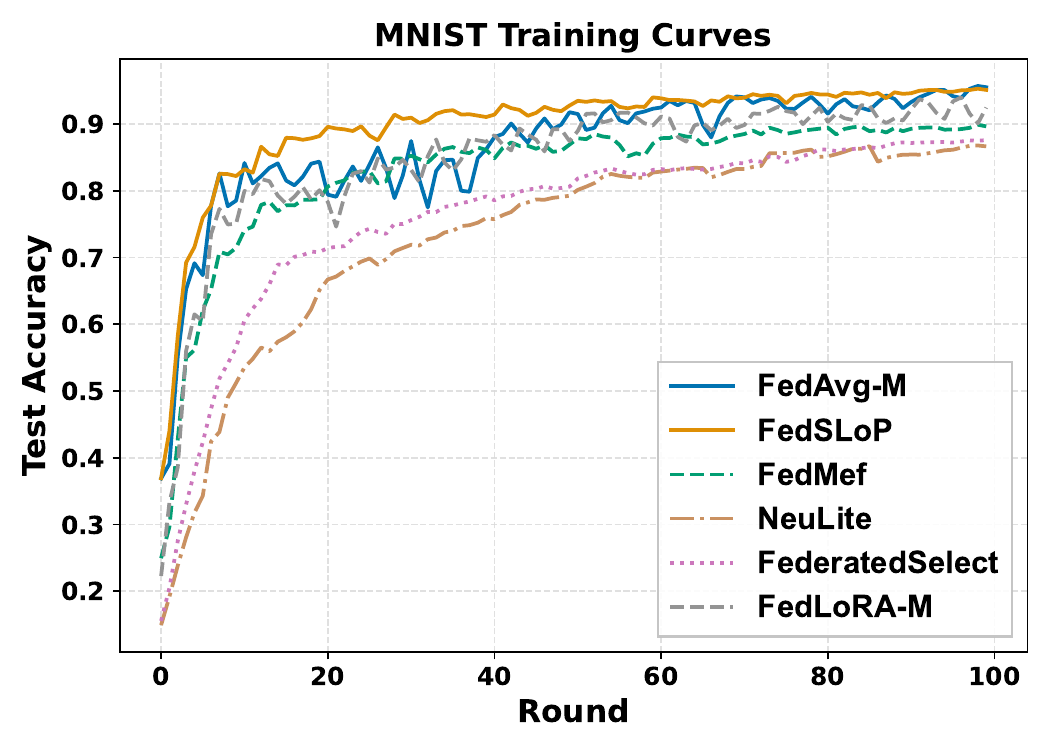}
  \end{minipage}
  \hfill
  \begin{minipage}{0.49\textwidth}
    \centering
    \includegraphics[width=\linewidth]{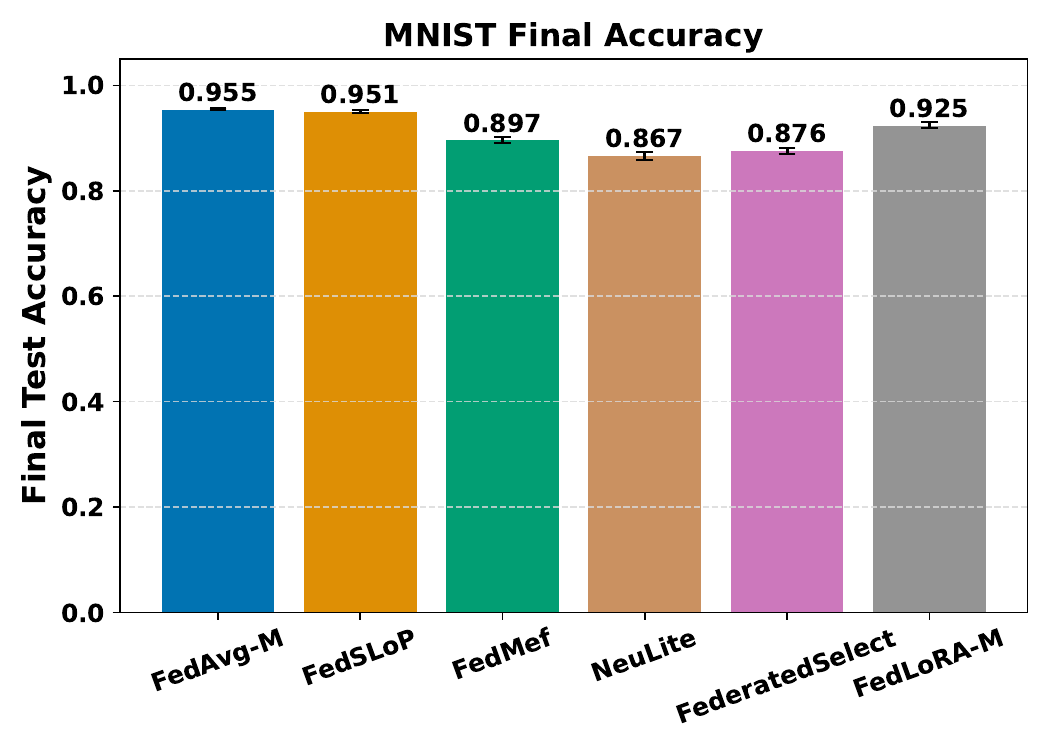}
  \end{minipage}
  \caption{Convergence and final precision on non-IID MNIST ($\alpha=0.1$). Left: evolution of test accuracy over $T=100$ communication rounds for all methods. FedSLoP closely tracks the strong baseline FedAvg-M, while sparse and partial-update methods converge more slowly and saturate at lower accuracies. Right: final test accuracy (mean $\pm$ standard deviation over three seeds).}
  \label{fig:mnist-conv-acc}
\end{figure}

\paragraph{Robustness to statistical heterogeneity.} To probe the sensitivity of each method to the strength of heterogeneity, we vary the Dirichlet concentration parameter $\alpha\in\{0.05,0.1,0.5,1.0\}$ while keeping all other hyperparameters fixed, and measure the final test accuracy at round $T=100$. Smaller $\alpha$ corresponds to more extreme label skew. Table~\ref{tab:mnist-alpha} and Figure~\ref{fig:mnist-alpha} summarize the results.

\begin{table}[t]
  \centering
  \caption{Final test accuracy on MNIST after $T=100$ communication rounds for different Dirichlet parameters $\alpha$ (single-seed results). Decreasing $\alpha$ increases statistical heterogeneity. FedSLoP degrades more gracefully than FedAvg-M and clearly outperforms sparse and low-rank model baselines in the strong non-IID regime.}
  \label{tab:mnist-alpha}
  \begin{tabular}{lcccc}
    \toprule
    Method & $\alpha=0.05$ & $\alpha=0.10$ & $\alpha=0.50$ & $\alpha=1.00$ \\
    \midrule
    FedAvg-M  & $0.9035$ & $\mathbf{0.9551}$ & $\mathbf{0.9722}$ & $\mathbf{0.9739}$ \\
    FedSLoP      & $\mathbf{0.9269}$ & ${0.9511}$ & $0.9700$ & $0.9708$ \\
    FedMef              & $0.7965$ & $0.8665$ & $0.9075$ & $0.9100$ \\
    NeuLite             & $0.7083$ & $0.7926$ & $0.8834$ & $0.8908$ \\
    FederatedSelect     & $0.8257$ & $0.8759$ & $0.9064$ & $0.9101$ \\
    FedLoRA-M & $0.8838$ & $0.9247$ & $0.9558$ & $0.9616$ \\
    \bottomrule
  \end{tabular}
\end{table}

\begin{figure}[t]
  \centering
  \includegraphics[width=0.6\textwidth]{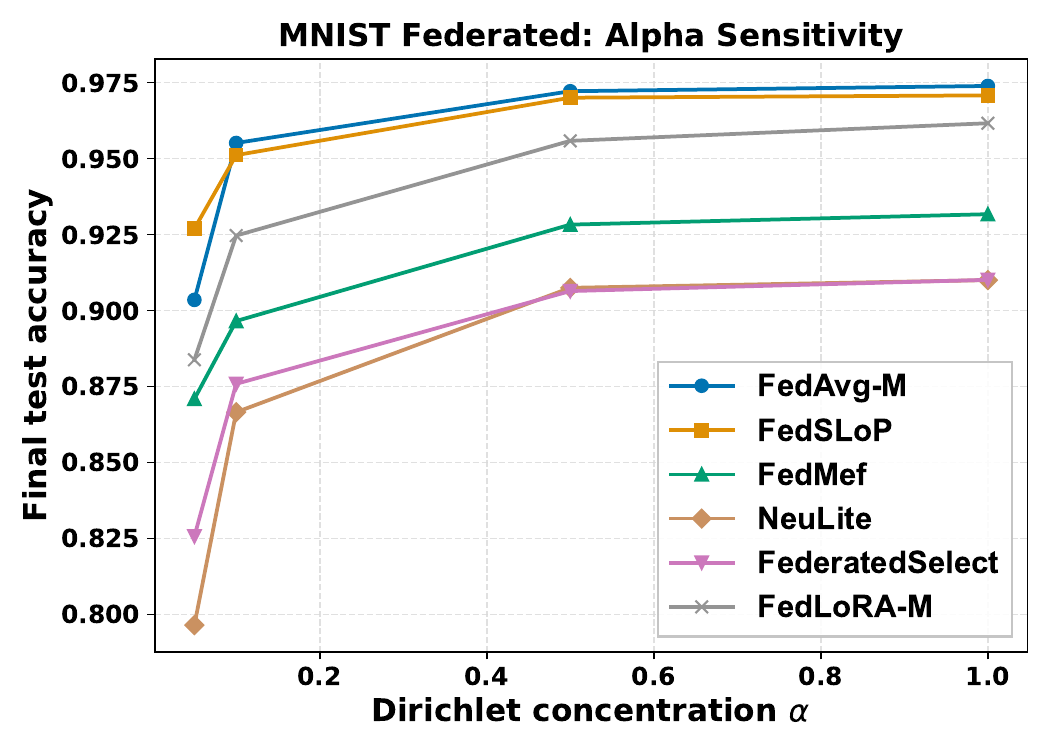}
  \caption{Sensitivity to statistical heterogeneity (Dirichlet $\alpha$). Final test accuracy after $T=100$ communication rounds as a function of the Dirichlet concentration parameter $\alpha$. Smaller $\alpha$ corresponds to stronger non-IIDness. FedSLoP exhibits a flatter degradation curve than FedAvg-M, maintaining $\approx 0.93$ accuracy even at $\alpha=0.05$, whereas FedAvg-M drops to $\approx 0.90$. Sparse and partial-update methods and the low-rank model baseline suffer much sharper degradation, especially at $\alpha=0.05$.}
  \label{fig:mnist-alpha}
\end{figure}

Two phenomena stand out. First, as $\alpha$ decreases from $1.0$ (nearly IID) to $0.05$ (extreme non-IID), all methods experience a drop in accuracy, but the magnitude of this drop varies substantially. FedSLoP degrades from $0.9709$ to $0.9269$, whereas FedAvg-M degrades from $0.9739$ to $0.9035$. Thus, in the most heterogeneous regime ($\alpha=0.05$), FedSLoP actually surpasses FedAvg-M by approximately $2\%$, despite being slightly inferior in the nearly IID regime. Second, the sparse and partial-update baselines exhibit much steeper degradation curves than both full-parameter methods.

\paragraph{Projection-rank ablation.} The projection rank $r$ is the key design parameter of FedSLoP, controlling the dimension of the subspace in which gradients are updated. To understand how $r$ affects the trade-off between communication cost and final precision, we perform a systematic ablation on MNIST with $\alpha=0.1$, varying $r\in\{16,32,64,112,192\}$ while keeping all other hyperparameters fixed. For each $r$, we record the final test accuracy at round $T=100$ and the average number of uploaded elements per round (a proxy for uplink communication cost). The results are summarized in Table~\ref{tab:rank-comm} and Figure~\ref{fig:rank-conv-tradeoff}.

\begin{table}[t]
  \centering
  \caption{Communication--accuracy trade-off for FedSLoP at different projection ranks $r$. The column ``Average uplink elements'' reports the mean number of parameters uploaded per round (proxy for communication cost). Increasing $r$ from $16$ to $112$ roughly triples the communication cost while improving accuracy by about $1.25$ percentage points; further increasing $r$ to $192$ yields negligible gains.}
  \label{tab:rank-comm}
  \begin{tabular}{lcc}
    \toprule
    Projection rank $r$ & Avg. uplink elements / round & Final accuracy (mean) \\
    \midrule
    $16$  & $2.35\times 10^4$ & $0.8907\pm0.0113$ \\
    $32$  & $4.55\times 10^4$ & $0.9148\pm0.0055$ \\
    $64$  & $8.97\times 10^4$ & $0.9392\pm0.0035$ \\
    $112$ & $1.56\times 10^5$ & $0.9511\pm0.0028$ \\
    $192$ & $2.60\times 10^5$ & $0.9561\pm0.0034$ \\
    \bottomrule
  \end{tabular}
\end{table}

\begin{figure}[t]
  \centering
  \includegraphics[width=0.49\textwidth]{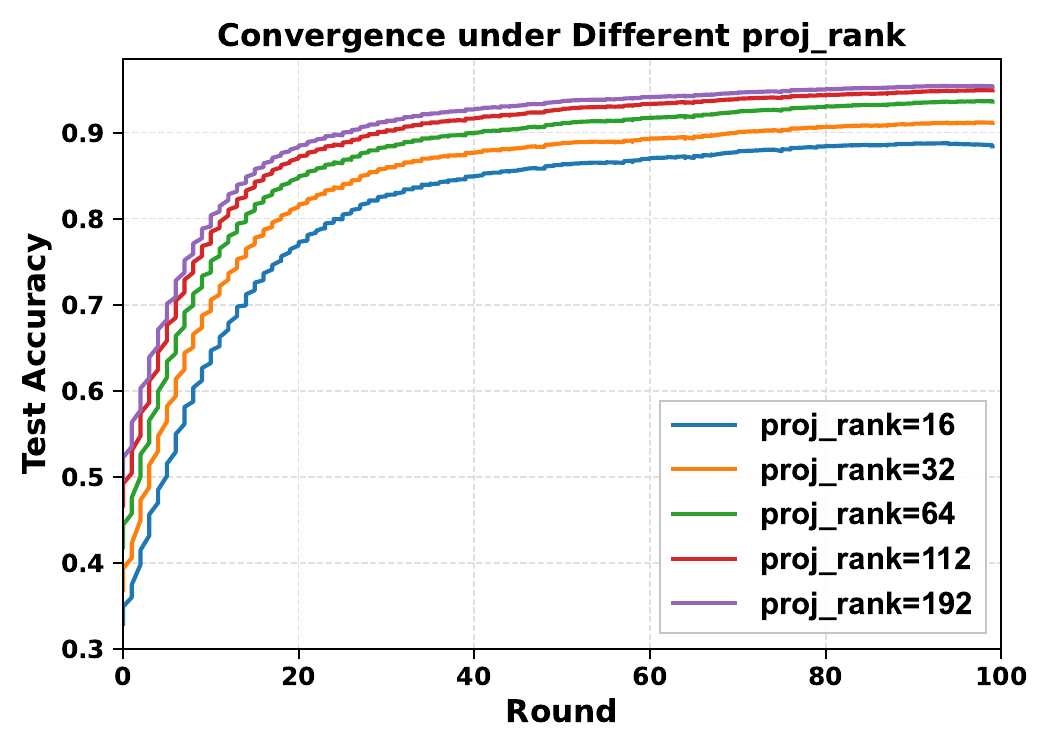}\includegraphics[width=0.49\textwidth]{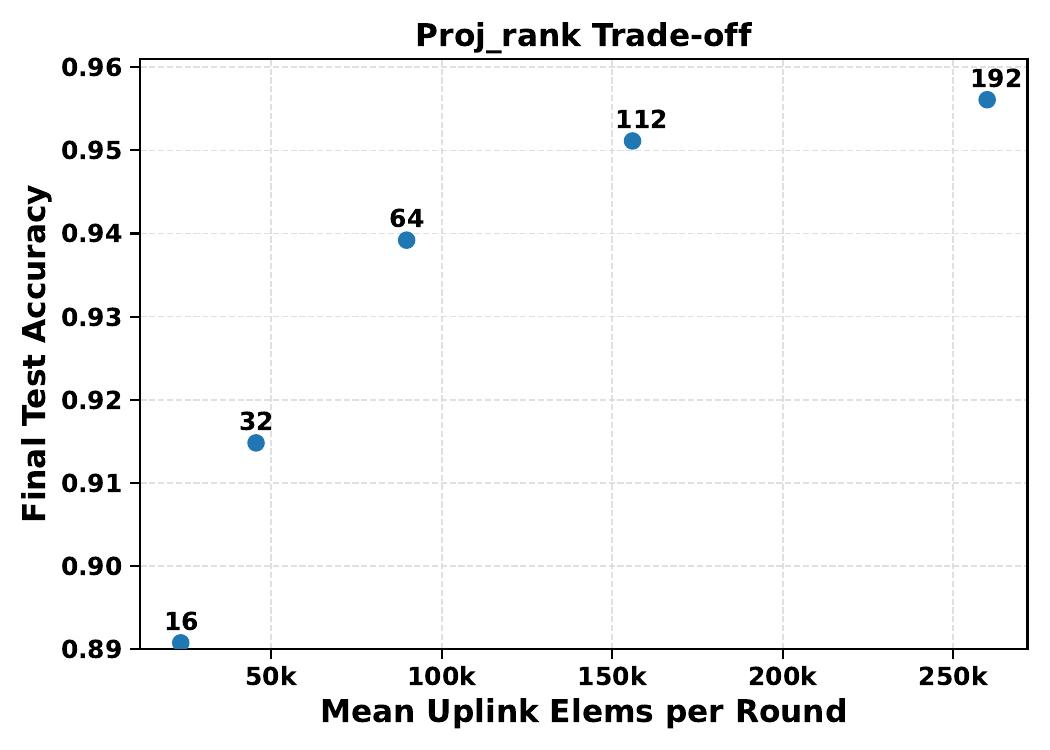}
  \caption{Projection-rank ablation for FedSLoP. Test accuracy vs. communication rounds for different projection ranks $r\in\{16,32,64,112,192\}$ on MNIST with $\alpha=0.1$. Larger $r$ accelerates convergence and improves final precision, with curves for $r\ge 112$ nearly overlapping, indicating diminishing returns beyond a certain rank.}
  \label{fig:rank-conv-tradeoff}
\end{figure}

The empirical trends align with the intuition that larger projection ranks reduce projection-induced bias but increase communication cost. As $r$ increases from $16$ to $64$, both convergence speed and final precision improve noticeably. However, beyond $r\approx 112$, the gains saturate: increasing $r$ from $112$ to $192$ yields only marginal improvements despite a substantial increase in communication cost.

\paragraph{Scaling with the number of clients.} The number of participating clients $p$ influences both the quality of the aggregated gradient and the communication cost per round. To examine how FedSLoP scales with $p$, we fix the projection rank at $r=112$ and vary $p\in\{10,20,50,100\}$ on MNIST with $\alpha=0.1$. For each $p$, we measure the final test accuracy at round $T=100$ and the relative total uplink cost per round (proportional to $p$). The results are reported in Table~\ref{tab:numclients-comm} and Figure~\ref{fig:numclients-conv-tradeoff}.

\begin{table}[t]
  \centering
  \caption{Communication--accuracy trade-off for FedSLoP at different numbers of clients $p$. The ``Total uplink (relative)'' column reports the total number of uploaded elements per round relative to the $p=10$ case.}
  \label{tab:numclients-comm}
  \begin{tabular}{lcc}
    \toprule
    Number of clients $p$ & Total uplink (relative) & Final accuracy (mean) \\
    \midrule
    $10$  & $1.0\times$  & $0.9199\pm0.0094$ \\
    $20$  & $2.0\times$  & $0.9341\pm0.0089$ \\
    $50$  & $5.0\times$  & $0.9472\pm0.0026$ \\
    $100$ & $10.0\times$ & $0.9479\pm0.0010$ \\
    \bottomrule
  \end{tabular}
\end{table}

\begin{figure}[t]
  \centering
  \includegraphics[width=0.49\textwidth]{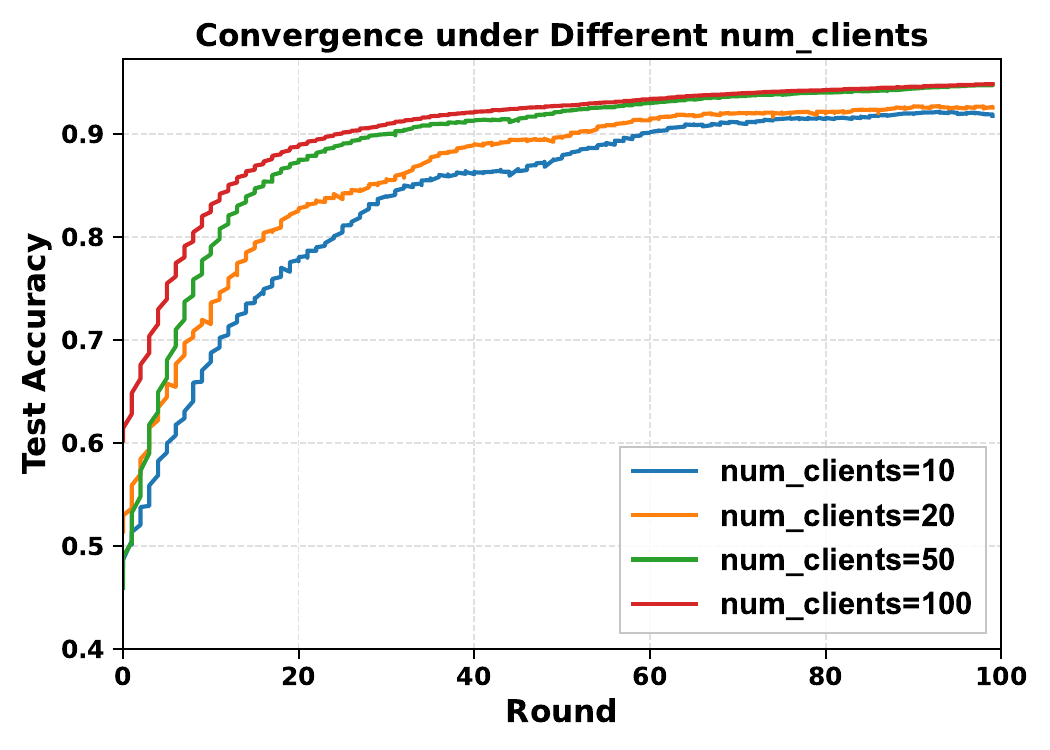}
  \includegraphics[width=0.49\textwidth]{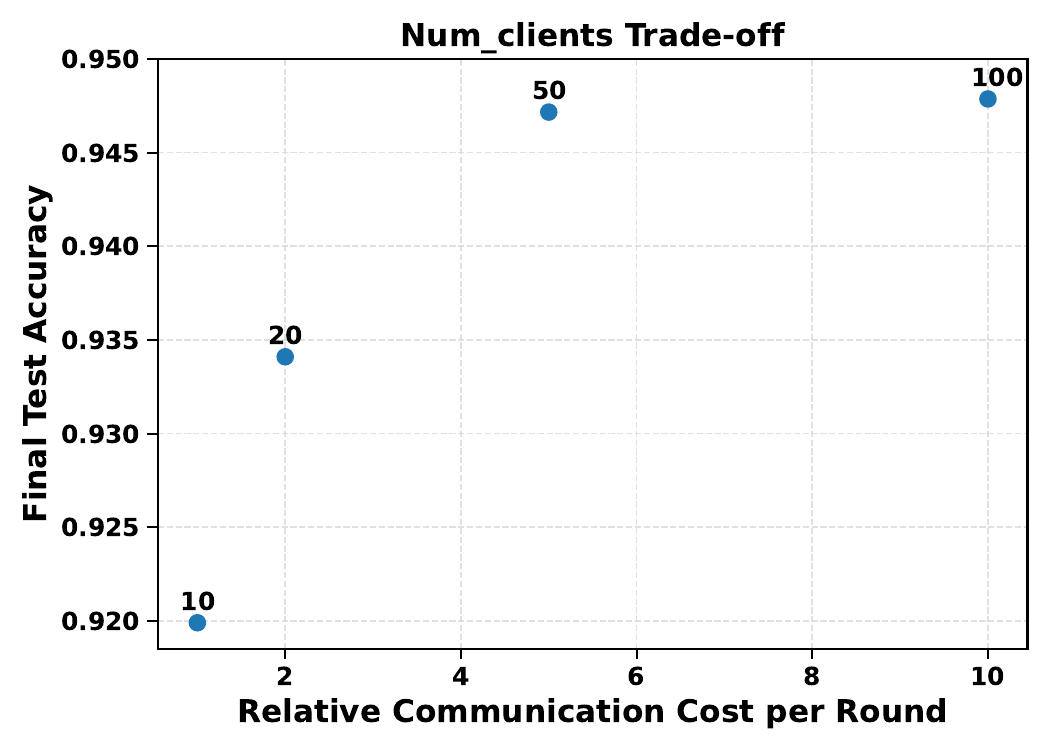}
  \caption{Client-number ablation for FedSLoP. Test accuracy vs. communication rounds for different numbers of clients $p\in\{10,20,50,100\}$ on MNIST with $\alpha=0.1$ and projection rank $r=112$. Larger $p$ accelerates convergence and reduces variance across rounds, reflecting improved gradient estimation, but exhibits diminishing returns beyond $p\approx 50$.}
  \label{fig:numclients-conv-tradeoff}
\end{figure}

Increasing $p$ from $10$ to $50$ reduces the variance of the aggregated gradient and leads to a noticeable improvement in both convergence speed and final precision. However, further increasing $p$ from $50$ to $100$ yields only a marginal gain, despite doubling the communication cost, indicating a clear diminishing-returns regime.

\paragraph{Summary.} Across all experiments, FedSLoP consistently strikes an attractive trade-off between communication and accuracy. On non-IID MNIST, it demonstrates convergence behavior comparable to the full-parameter FedAvg-M baseline while substantially outperforming sparse and low-rank model baselines in both final precision and robustness to heterogeneity. The projection-rank and client-number ablations further highlight how the subspace dimension $r$ and the number of clients $p$ jointly control the balance between efficiency and performance, providing practical guidance for deploying FedSLoP in resource-constrained federated systems.
}{}
\IfFileExists{conclusion.tex}{\section{Conclusion}

We have proposed FedSLoP, a FedAvg-style federated learning algorithm that integrates random low-rank subspace projections with momentum in order to reduce client-side memory usage while retaining the simplicity and robustness of the standard FL workflow. By sampling a fresh Stiefel subspace at each communication round and constraining all local stochastic gradients and momentum updates to this subspace, FedSLoP reduces the effective dimensionality of the optimizer state and the communicated directions without modifying the underlying model architecture.

On the theoretical side, we developed a detailed nonconvex convergence analysis under standard assumptions, deriving explicit bounds on the averaged squared gradient norm that exhibit an order of  $\mathcal{O}(1/\sqrt{NT})$ with linear speedup in the number of clients $N$. 

On the empirical side, our experiments on federated MNIST classification with Dirichlet non-IID partitions showed that FedSLoP closely tracks the convergence of a strong FedAvg-M baseline while substantially outperforming sparse and low-rank model baselines such as FedMef, NeuLite, FederatedSelect, and FedLoRA-M. Projection-rank and client-number ablations revealed clear diminishing-returns regimes and confirmed that moderate ranks and client counts already suffice to capture the dominant gradient subspace and achieve stable, accurate training under severe heterogeneity.

Taken together, these results demonstrate that constraining the optimizer state to random low-dimensional subspaces is a principled and effective strategy for communication- and memory-efficient federated learning. Several directions remain for future work. Extending the convergence analysis to partial participation and asynchronous settings, incorporating adaptive or data-dependent subspace selection, and applying FedSLoP to larger-scale models and tasks (e.g., language modeling and vision transformers) are particularly promising. Another avenue is to combine the subspace-based design with advanced compression and personalization techniques, further closing the gap between theoretical guarantees and the demanding system constraints of real-world FL deployments.
}{}

\section*{Acknowledgments}
We gratefully acknowledge ReasFlow~\cite{reasflowteam2026reasflow}, a reasoning-centric scientific discovery assistant, for its substantial contributions to the preparation of this paper. A significant portion of the work, including the literature review, mathematical proofs, numerical experiments, and the initial manuscript draft, was generated automatically with the assistance of ReasFlow. The authors’ contributions lay primarily in identifying the research problem, confirming the algorithmic design, specifying the settings for the numerical experiments, and polishing the manuscript to meet the standards required for submission. In particular, the authors devoted considerable effort to verifying the correctness of the mathematical proofs and refining the resulting arguments.

\IfFileExists{references.bib}{
	\bibliographystyle{plain}
	\bibliography{references}
}{
	\IfFileExists{references_inline.tex}{
\bibliographystyle{plainnat}
\bibliography{references}
}{
		
	}
}

\end{document}